\documentclass{new_tlp}
\usepackage{mathptmx}
\usepackage{mathtools}

\usepackage[capitalize,nameinlink]{cleveref}

\DeclareMathAlphabet{\mathcal}{OMS}{cmsy}{m}{n}

\usepackage{import}

\usepackage{ifthen}
\usepackage{url}
\usepackage{amssymb}
\usepackage{amsmath}
\usepackage{import}
\usepackage{xparse}
\usepackage{amsmath}
\usepackage[usenames,dvipsnames]{xcolor}
\usepackage{xspace}
\usepackage{datatool}
\usepackage{glossaries}

\providecommand\m[1]{\ensuremath{#1}\xspace}
\renewcommand{\m}[1]{\ensuremath{#1}\xspace}
\newcommand{\trval}[1]{\m{\mathbf{#1}}}

	\newcommand{\lrule}{\leftarrow}
	\newcommand{\cause}{\stackrel{c}{\lrule}}
	
	\newcommand{\ltrue}{\trval{t}}
	\newcommand{\lfalse}{\trval{f}}
	\newcommand{\lunkn}{\trval{u}}
	
	\newcommand{\Tr}{\ltrue}
	\newcommand{\Fa}{\lfalse}
	\newcommand{\Un}{\lunkn}

	\newcommand{\struct}{\m{I}}

	\newcommand{\rules}{\m{R}}
	
	\NewDocumentCommand\inter{g+g}{%
	  \IfNoValueTF{#1}
	    {\struct}
	    {\m{#1^{#2}}}}

	\renewcommand{\int}{\m{\mathbb{Z}}}

	\newcommand{\leqt}{\m{\leq_t}}

	\NewDocumentCommand\subs{g+g}{%
	  \IfNoValueTF{#1}
	    {\m{/}}
	    {\m{#1/ #2}}}

\newcommand{\ouracronym}[3]{%
	\newacronym{#1}{#2}{#3}
	\expandafter\newcommand\csname #1\endcsname{\gls{#1}\xspace}%
}
	\ouracronym{FO}{FO}{first-order logic}
	\ouracronym{PC}{PC}{propositional calculus}
	\ouracronym{MX}{MX}{Model Expansion}
	\ouracronym{MO}{MO}{Model Optimization}
	\ouracronym{ASP}{ASP}{Answer Set Programming}
	\ouracronym{TP}{TP}{Theorem Proving}
	\ouracronym{LP}{LP}{Logic Programming}
	\ouracronym{CP}{CP}{Constraint Programming}
	\ouracronym{FP}{FP}{Functional Programming}
	\ouracronym{KR}{KR}{Knowledge Representation}
	\ouracronym{CSP}{CSP}{Constraint Satisfaction Problem}
	\ouracronym{SMT}{SMT}{SAT Modulo Theories}
	\ouracronym{KBS}{KBS}{knowledge base system}
	\ouracronym{NNF}{NNF}{Negation Normal Form}
	\ouracronym{FNNF}{FNNF}{Flat Negation Normal Form}
	\ouracronym{DefNNF}{DefNNF}{Definition Negation Normal Form}
	\ouracronym{DEFNF}{DEFNF}{Definition Normal Form}
	\ouracronym{CDCL}{CDCL}{Conflict-Driven Clause-Learning}
	\ouracronym{WFS}{WFS}{Well-Founded Semantics}
	\ouracronym{LCG}{LCG}{Lazy Clause Generation}
	\ouracronym{AEL}{AEL}{Autoepistemic Logic}
	\ouracronym{OEL}{OEL}{Ordered Epistemic Logic}
	\ouracronym{AFT}{AFT}{Approximation Fixpoint Theory}

	\makeatletter
	\def\ifenv#1{
	\def\@tempa{#1}%
	\def\@ttempa{#1*}%
	\ifx\@tempa\@currenvir
	\expandafter\@firstoftwo
	\else
	\expandafter\@secondoftwo
	\fi
	}
	\makeatother

	\newcommand{\ddrule}[4]{\ensuremath{#1 \leftarrow #2 & \{#3\} & #4}}
	\newcommand{\drule}[2]{\ensuremath{#1 & \leftarrow & #2}}

	\newcommand{\darule}[4]{\ensuremath{#1 \leftarrow #2 & \{#3\} & #4}}
	\newcommand{\arule}[2]{\ensuremath{#1 \, &\leftarrow \, #2}}

	\newcommand{\LNDRule}[2]{
	\ifenv{array}
	{\drule{#1}{#2}}
	{ \ifenv{align}
		{\arule{#1}{#2}}
		{\ifenv{align*}
		{\arule{#1}{#2}}
		{ERROR: using LDRule in unsupported environment: \@currenvir}
		}
	}
	}

	\newcommand{\LDRule}[4]{
	\ifenv{array}
	{\ddrule{#1}{#2}{#3}{#4}}
	{ \ifenv{align}
		{\darule{#1}{#2}{#3}{#4}}
		{\ifenv{align*}
		{\darule{#1}{#2}{#3}{#4}}
		{ERROR: using LDRule in unsupported environment: \@currenvir}
		}
	}
	}

	\NewDocumentCommand\LRule{m+g+g+g}{%
		\IfNoValueTF{#2}%
		{#1.&}{%
		\IfNoValueTF{#3}
		{\LNDRule{#1}{#2.}}
		{\LDRule{#1}{#2.}{#3}{#4}}%
		}
	}

	\NewDocumentCommand\CLRule{m+g}{%
	\ifenv{array}
	{\cdrule{#1}{#2}}
	{ \ifenv{align}
		{\carule{#1}{#2}}
		{\ifenv{align*}
			{\carule{#1}{#2}}
			{ERROR: using CLRule in unsupported environment: \@currenvir}
		}
	}
	}

	\NewDocumentCommand\carule{m+g}{%
		\IfNoValueTF{#2}
			{\ensuremath{#1.}}
			{\ensuremath{#1 \, &\cause \, #2}}}
	\NewDocumentCommand\cdrule{m+g}{%
		\IfNoValueTF{#2}
			{\ensuremath{#1.}}
			{\ensuremath{#1 & \cause & #2}}}

	\newcommand{\algrule}[4]{
	\hbox{{#1}:}& 
	\quad #2 ~\longrightarrow~ #3 
	\hbox{~ if } #4\\
	}

	\newcommand{\AlgoRule}[4]{
	\ifenv{array}
	{\algrule{#1}{#2}{#3}{#4}}
		{ERROR: using AlgoRule in unsupported environment: \@currenvir}
	}

	\newcommand{\ignore}[1]{}

	\newboolean{nocomments}
	\setboolean{nocomments}{false}

	\newboolean{commentmargin}
	\setboolean{commentmargin}{true}

	\newcommand{\namedcomment}[3]{%
		\ifthenelse{\boolean{nocomments}}%
		{}%IF no comments, write nothing
		{%Otherwise
			\ifthenelse{\boolean{commentmargin}}%
				{ {\color{#3} \marginpar{\color{#3}\sc #2}#1}  }%Name in margin
				{  {\color{#3} {\sc #2}: #1}  }%Name not in margin
		}%
	}
	\newcommand{\mnamedcomment}[3]{\ifthenelse{\boolean{nocomments}}{}{{\marginpar{ \color{#3}{\sc #2}:#1}}}}

	\usepackage{soul}

\usepackage[normalem]{ulem} % wavy underlines
\makeatletter
\font\uwavefont=lasyb10 scaled 700
\def\spelling{\bgroup\markoverwith{\lower3.5\p@\hbox{\uwavefont\textcolor{Red}{\char58}}}\ULon}
\def\grammar{\bgroup\markoverwith{\lower3.5\p@\hbox{\uwavefont\textcolor{LimeGreen}{\char58}}}\ULon}
\def\phrasing{\bgroup\markoverwith{\lower3.5\p@\hbox{\uwavefont\textcolor{RoyalBlue}{\char58}}}\ULon}

\newcommand\remove{\bgroup\markoverwith{\textcolor{red}{\rule[0.5ex]{2pt}{0.4pt}}}\ULon}
\makeatother

\usepackage{etoolbox}

\newcommand\setcitation[2]{%
  \csdef{mycommoncitation#1}{#2}}
\newcommand\getcitation[1]{%
  \csuse{mycommoncitation#1}}

\setcitation{IDP}{WarrenBook/DeCatBBD14}
\setcitation{idp}{WarrenBook/DeCatBBD14}
\setcitation{fodot}{tocl/DeneckerT08}
\setcitation{foid}{tocl/DeneckerT08}
\setcitation{FOID}{tocl/DeneckerT08}
\setcitation{cplogic}{journal/tplp/VennekensDB10}
\setcitation{CPlogic}{journal/tplp/VennekensDB10}
\setcitation{CPLogic}{journal/tplp/VennekensDB10}
\setcitation{CP}{fai/Rossi06}
\setcitation{cp}{fai/Rossi06}
\setcitation{EZCSP}{lpnmr/Balduccini11}
\setcitation{KR}{Baral:2003}
\setcitation{ASPComp2}{lpnmr/DeneckerVBGT09}
\setcitation{ASPComp3}{journals/tplp/CalimeriIR14}
\setcitation{ASPComp4}{conf/lpnmr/AlvianoCCDDIKKOPPRRSSSWX13}
\setcitation{ASPComp5}{journals/ai/CalimeriGMR16}
\setcitation{ASPComp6}{jair/GebserMR17}
\setcitation{ASPComp7}{tplp/GebserMR20}
\setcitation{CPSupport}{ictai/DeCat13}
\setcitation{CPsupport}{ictai/DeCat13}
\setcitation{functionDetection}{iclp/DeCatB13}
\setcitation{FunctionDetection}{iclp/DeCatB13}
\setcitation{fodot2asp}{DeneckerLTV12} %TODO replace by journal publication if one is publishedr
\setcitation{Tarskian}{DeneckerLTV12} %Same as the one above 
\setcitation{TarskianSemanticsASP}{DeneckerLTV12} %Same as the one above 
\setcitation{Inca}{iclp/DrescherW12}
\setcitation{csp2asp}{ijcai/DrescherW11}
\setcitation{DPLLT}{cav/GanzingerHNOT04}
\setcitation{AspInPractice}{synthesis/2012Gebser}
\setcitation{ASPInPractice}{synthesis/2012Gebser}
\setcitation{clasp}{ai/GebserKS12}
\setcitation{oclingo}{kr/GebserGKOSS12}
\setcitation{clingo}{iclp/GebserKKOSW16}
\setcitation{gringo}{lpnmr/GebserST07}
\setcitation{cmodels}{aaai/GiunchigliaLM04}
\setcitation{inputster}{tplp/Jansen13}
\setcitation{DLV}{tocl/LeonePFEGPS06}
\setcitation{LearningPaper}{TPLP/BruynoogheBBDDJLRDV} %TODO replace by published version
\setcitation{clog}{iclp/BogaertsVDV14} %TODO replace by journal publication if one is published
\setcitation{foc}{iclp/BogaertsVDV14}
\setcitation{FOC}{iclp/BogaertsVDV14}
\setcitation{inferenceClog}{ecai/BogaertsVDV14}
\setcitation{examplesClog}{nmr/BogaertsVDV14b} %TODO replace by better publication if possible
\setcitation{AFT}{DeneckerMT00}
\setcitation{KBS}{iclp/DeneckerV08}
\setcitation{KBS-invitedtalk}{jelia/Denecker16}
\setcitation{KBPE}{inap/DePooterWD11}
\setcitation{lazyGrounding}{jair/CatDBS15} 
\setcitation{LazyGrounding}{jair/CatDBS15} 
\setcitation{lazygrounding}{jair/CatDBS15} 
\setcitation{lazygroundingASP}{ijcai/BogaertsW18} 
\setcitation{justifications}{lpnmr/DeneckerBS15} 
\setcitation{justificationsAlpha}{ijcai/BogaertsW18} 
\setcitation{ASP}{marek99stable}
\setcitation{satid}{sat/MarienWDB08}
\setcitation{lazyclausegeneration}{constraints/OhrimenkoSC09}
\setcitation{FP}{ACMCS/Hudak89}
\setcitation{GroundingWithBounds}{jair/WittocxMD10}
\setcitation{GroundWithBounds}{jair/WittocxMD10}
\setcitation{SAT}{faia/SilvaLM09}
\setcitation{HandbookOfSAT}{faia/2009-185}
\setcitation{LTC}{iclp/Bogaerts14}
\setcitation{SPSAT}{ictai/DevriendtBMDD12}
\setcitation{BreakID}{sat/DevriendtBBD16}
\setcitation{breakid}{sat/DevriendtBBD16}
\setcitation{LCG}{stuckeyLCG}
\setcitation{MiniZinc}{conf/cp/NethercoteSBBDT07}
\setcitation{minizinc}{conf/cp/NethercoteSBBDT07}
\setcitation{amadini}{cpaior/AmadiniGM13}
\setcitation{bootstrapping}{ngc/BogaertsJDJBD16}
\setcitation{Bootstrapping}{ngc/BogaertsJDJBD16}
\setcitation{GroundedFixpoints}{ai/BogaertsVD15}
\setcitation{PartialGroundedFixpoints}{ijcai/BogaertsVD15}
\setcitation{LogicBlox}{datalog/GreenAK12}
\setcitation{proB}{journals/sttt/LeuschelB08}
\setcitation{NaturalInductions}{KR/DeneckerV14} %TODO replace by journal publication if one is published
\setcitation{LP}{jacm/EmdenK76}
\setcitation{SMT}{faia/BarrettSST09}
\setcitation{AF}{ai/Dung95}
\setcitation{ADF}{kr/BrewkaW10}
\setcitation{af}{ai/Dung95}
\setcitation{adf}{kr/BrewkaW10}
\setcitation{ADFRevisited}{ijcai/BrewkaSEWW13}
\setcitation{adfrevisited}{ijcai/BrewkaSEWW13}
\setcitation{DefaultLogic}{ai/Reiter80}
\setcitation{DL}{ai/Reiter80}
\setcitation{AEL}{mo85}
\setcitation{minisat}{sat/EenS03}
\setcitation{completion}{adbt/Clark78}
\setcitation{ClarkCompletion}{adbt/Clark78}
\setcitation{wasp}{lpnmr/AlvianoDFLR13}
\setcitation{minisatid}{ictai/DeCat13}
\setcitation{lcg}{stuckeyLCG}
\setcitation{CEGAR}{jacm/ClarkeGJLV03}
\setcitation{cegar}{jacm/ClarkeGJLV03}
\setcitation{CuttingPlane}{or/DantzigFJ54}
\setcitation{kodkod}{tacas/TorlakJ07}
\setcitation{cdcl}{Marques-SilvaS99}
\setcitation{CDCL}{Marques-SilvaS99}
\setcitation{1UIP}{iccad/ZhangMMM01}
\setcitation{relevance}{ijcai/JansenBDJD16}
\setcitation{relevance-implementation}{aspocp/JansenBDJD16}
\setcitation{WFS}{GelderRS91}
\setcitation{wfs}{GelderRS91}
\setcitation{UnfoundedSet}{GelderRS91}
\setcitation{UFS}{GelderRS91}
\setcitation{stablesemantics}{iclp/GelfondL88}
\setcitation{StableSemantics}{iclp/GelfondL88}
\setcitation{shatter}{Shatter}
\setcitation{sbass}{drtiwa11a}
\setcitation{lparsemanual}{url:lparse_manual}
\setcitation{AIC}{ppdp/FlescaGZ04}
\setcitation{templates}{tplp/DassevilleHJD15}
\setcitation{templates2}{iclp/DassevilleHBJD16}
\setcitation{sat-to-sat}{aaai/JanhunenTT16}
\setcitation{sat-to-sat-qbf}{bnp/BogaertsJT16}
\setcitation{sat-to-sat-QBF}{bnp/BogaertsJT16}
\setcitation{sat-to-sat-SO}{kr/BogaertsJT16}
\setcitation{XSB}{SwiW12}
\setcitation{KCmap}{jair/DarwicheM02}
\setcitation{TLA}{DBLP:books/aw/Lamport2002}
\setcitation{EventB}{BookAbrial2010}
\setcitation{MX}{MitchellT05}
\setcitation{MIP}{Sierksma96}
\setcitation{perefectmodel}{minker88/Przymusinski88}
\setcitation{SafeInductions}{ijcai/BogaertsVD17}
\setcitation{AIC}{ppdp/FlescaGZ04}
\setcitation{aic}{ppdp/FlescaGZ04}
\setcitation{alpha}{lpnmr/Weinzierl17}
\setcitation{omiga}{jelia/Dao-TranEFWW12}
\setcitation{gasp}{fuin/PaluDPR09}
\setcitation{asperix}{lpnmr/LefevreN09a}
\setcitation{CTL}{lop/ClarkeE81}
\setcitation{AFT-AIC}{ai/BogaertsC18}
\setcitation{UltimateApproximator}{DeneckerMT04}
\setcitation{KripkeKleene}{Fitting85}
\setcitation{AFT-HO}{corr/CharalambidisRS18} %TODO replace
\setcitation{HereThere}{Heyting30}
\setcitation{dAEL}{ijcai/HertumCBD16}
\setcitation{SDD}{ijcai/Darwiche11}
\setcitation{HEX}{ijcai/EiterIST05}
\setcitation{wADF}{aaai/BrewkaSWW18}
\setcitation{wADFfix}{corr/BrewkaSWW18}
\setcitation{TransitionSystems}{jacm/NieuwenhuisOT06}
\setcitation{galliwasp}{lopstr/MarpleG12}
\setcitation{GalliWasp}{lopstr/MarpleG12}
\setcitation{clingcon}{tplp/BanbaraKOS17}
\setcitation{lp2sat}{birthday/JanhunenN11}
\setcitation{lp2mip}{LIU12}
\setcitation{lp2diff}{lpnmr/JanhunenNS09}
\setcitation{lp2acyc}{ecai/GebserJR14}
\setcitation{PB}{faia/RousselM09}
\setcitation{pb}{faia/RousselM09}
\setcitation{CuttingPlanes}{dam/CookCT87}
\setcitation{RoundingSAT}{ijcai/ElffersN18}
\setcitation{PRS}{aaai/DixonG02}
\setcitation{sat4j}{jsat/BerreP10}
\setcitation{SAT4J}{jsat/BerreP10}
\setcitation{mingo}{LIU12}
\setcitation{pbmodels}{lpnmr/LiuT05}
\setcitation{HEF-LP}{lpnmr/GebserLL07}
\setcitation{HCF-LP}{amai/Ben-EliyahuD94}
\setcitation{aspcore2}{AspCore2}
\setcitation{}{}
\setcitation{}{}
\setcitation{}{}
\setcitation{}{}
\setcitation{}{}
\setcitation{}{}
\setcitation{}{}
\setcitation{}{}
\setcitation{}{}
\setcitation{}{}
\setcitation{}{}
\setcitation{}{}
\setcitation{}{}
\setcitation{}{}
\setcitation{}{}
\setcitation{}{}
\setcitation{}{}
\setcitation{}{}
\setcitation{}{}
\setcitation{}{}
\setcitation{}{}
\setcitation{}{}
\setcitation{}{}
  
\newcommand\mycite[1]{%
      \ifcsname mycommoncitation#1\endcsname%
   \cite{\getcitation{#1}}%
  \else%
    \cite{#1}%
  \fi%
}	
  
\newcommand\mycitet[1]{%
      \ifcsname mycommoncitation#1\endcsname%
   \citet{\getcitation{#1}}%
  \else%
    \citet{#1}
  \fi%
}

\ignore{

}
\usepackage{tikz}
\usepackage{tkz-berge}
\usepackage{pdfpages}

\setboolean{nocomments}{false}

\newcommand\setl[1]{\left\lbrace #1 \right\rbrace}
\newcommand\setprop[2]{\left\lbrace #1 \; \middle| \; #2 \right\rbrace}
\newcommand{\tild}{\mathord{\sim}}

\newcommand{\suppletter}{\mathcal{S}}

\newcommand{\suppopsys}[1]{\suppletter_{#1}}
\newcommand{\supp}{\suppopsys{\js}}

\newcommand{\interp}{\mathcal{I}}
\newcommand{\be}{\mathcal{B}}
\newcommand{\besub}[1]{\m{\be_{\mathrm{#1}}}}
\newcommand{\bekk}{\besub{KK}}
\newcommand{\bewf}{\besub{wf}}
\newcommand{\best}{\besub{st}}
\newcommand{\besp}{\besub{sp}}

\newcommand{\F}{\mathcal{F}}
\newcommand{\jf}{{\m{\mathcal{J}\kern-0.2em\mathcal{F}}}}

\newcommand{\jfcomplete}[1][\rules]{\left\langle \F, \Fd, #1\right\rangle}

\newcommand{\js}{{\mathcal{J}\kern-0.2em\mathcal{S}}}
\newcommand{\jscomplete}[1][\be]{\left\langle \F, \Fd, \rules, #1\right\rangle}
\newcommand{\lf}{\mathcal{L}}

\newcommand{\Fp}{{\m{\F_{+}}}}
\newcommand{\Fn}{{\m{\F_{-}}}}
\newcommand{\Fd}{{\m{\F_d}}}
\newcommand{\Fo}{{\m{\F_o}}}

\newcommand{\pathstyle}[1]{\mathbf{#1}}
\newcommand{\branch}{\m{\pathstyle{b}}}

\newcommand{\branches}{\m{B}}
\newcommand{\justifications}{\m{\mathfrak{J}}}

\newcommand{\ptr}{\m{\trval{T}}}
\newcommand{\pfa}{\m{\trval{F}}}
\newcommand{\pref}{\m{\sqsubseteq}}

\newcommand{\prefplayer}[1]{\m{\pref_{#1}}}

\newcommand{\preftr}{\m{\prefplayer{\ptr}}}
\newcommand{\preffa}{\m{\prefplayer{\pfa}}}

\newcommand{\graph}{\m{G}}
\newcommand{\game}{\m{\mathcal{G}}}
\newcommand{\states}{\m{S}}
\newcommand{\edges}{\m{E}}

\newcommand{\paths}[1]{\m{\pathtext_{#1}}}
\newcommand{\play}[1]{\m{p_{#1}}}
\newcommand{\strategies}{\m{\mathfrak{S}}}
\newcommand{\gstrategies}{\m{\strategies_g}}
\newcommand{\pstrategies}{\m{\strategies_p}}

\newcommand\myparagraph[1]{

\noindent\emph{#1}\quad}

\ouracronym{BE}{BE}{branch evaluation}

\DeclareMathOperator{\im}{Im}

\DeclareMathOperator{\sgn}{sgn}
\DeclareMathOperator{\suppvalue}{SV}
\DeclareMathOperator{\jval}{val}
\DeclareMathOperator{\source}{source}
\DeclareMathOperator{\target}{target}
\DeclareMathOperator{\plays}{Play}
\DeclareMathOperator{\pathtext}{Path}

\title
{Exploiting Game Theory for Analysing Justifications}

\author[S. Marynissen, B. Bogaerts and M. Denecker]{SIMON MARYNISSEN$^{1,2}$, BART BOGAERTS$^2$, MARC DENECKER$^1$ \\
$^1$KU Leuven \quad $^2$Vrije Universiteit Brussel}

\jdate{May 2020}
\pubyear{2020}
\pagerange{\pageref{firstpage}--\pageref{lastpage}}
\doi{doi}

\newtheorem{theorem}{Theorem}[section]
\newtheorem{lemma}[theorem]{Lemma}
\newtheorem{definition}[theorem]{Definition}
\newtheorem{proposition}[theorem]{Proposition}
\newtheorem{remark}[theorem]{Remark}
\newtheorem{example}[theorem]{Example}
\newtheorem{corollary}[theorem]{Corollary}

\usepackage{pgffor}

\newcounter{numberOfStoredProofs}
\stepcounter{numberOfStoredProofs}

\newboolean{arxiv}
\setboolean{arxiv}{true}

\newcommand{\arxivpaper}[1]{\ifthenelse{\boolean{arxiv}}{#1}{}}
\newcommand{\tplppaper}[1]{\ifthenelse{\boolean{arxiv}}{}{#1}}

\newboolean{displayproof}
\newcommand\proofintext[1]{#1}
\renewcommand\proofintext[1]{}
\newcommand\thmwithproof[5]{%
  \thmwithproofgeneral{#1}{#2}{#3}{#4}{#5}{false}%
}

\newcommand\lemmaInAppendix[1]{%
\expandafter\newcommand\csname mystoredproof\the\value{numberOfStoredProofs} \endcsname{#1}\stepcounter{numberOfStoredProofs}}

\newcommand\thmwithproofgeneral[6]{%
  \begin{#3}\label{#1}#4%
  \end{#3}%
  \setboolean{displayproof}{#6}%
  \ifthenelse{\boolean{displayproof}}%
  {\begin{proof}#5\end{proof}}%
{ \proofintext{ \begin{proof} 	#5 \end{proof}}%
    \expandafter\newcommand\csname mystoredproof\the\value{numberOfStoredProofs} \endcsname{\noindent {\it #2 \ref{#1}}% 
      \noindent #4 \begin{proof}#5\end{proof}}%
    \stepcounter{numberOfStoredProofs}%
  }%
}
\newcommand\getproof[1]{
  \ifcsname mystoredproof#1 \endcsname%
  \csname mystoredproof#1 \endcsname%
  \else%
  \fi%
}

\newcommand\proofsketch[1]{
  \ifthenelse{\boolean{displayproof}}
  {}
  {#1}
  
}

\newcommand{\proofs}{
  \foreach \n in {0,...,\value{numberOfStoredProofs}}{\getproof{\n}}
}

\newcommand\citet[1]{\citeNS{#1}\xspace}
\begin{document}
\setcounter{page}{1}
\label{firstpage}
\maketitle

\begin{abstract}
  Justification theory is a unifying semantic framework.
  While it has its roots in non-monotonic logics, it can be applied to various areas in computer science, especially in explainable reasoning; its most central concept is a justification: an explanation why a  property holds (or does not hold) in a model.
  
  In this paper, we continue the study of justification theory by means of three major contributions.
  The first is studying the relation between justification theory and game theory.
  We show that justification frameworks can be seen as a special type of games.
  The established connection provides the theoretical foundations for our next two contributions.
  The second contribution is studying under which condition two different dialects of justification theory (graphs as explanations vs trees as explanations) coincide.
  The third contribution is establishing a precise criterion of when a semantics induced by justification theory yields consistent results.
  In the past proving that such semantics were consistent took cumbersome and elaborate proofs.
  
  We show that these criteria are indeed satisfied for all common semantics of logic programming.
  
  \arxivpaper{This paper is under consideration for acceptance in Theory and Practice of Logic Programming (TPLP).}
\end{abstract}


\section{Introduction}

Justification theory is an abstract theory to define semantics of non-monotonic logics. 
It is a powerful and versatile framework: 
First, it provides a mechanism to define new logics based on well-known principles, and to transfer results between domains. 
Second, it brings order in the zoo of logics and semantics, on the one hand by allowing a systematic comparison between multiple semantics for a single logic and on the other hand by enabling a comparison between different logics (which semantics of which logic coincides with semantics of another logic?). 
Third, it enables modular definitions of semantics.
Using so-called nesting of justification frames, efforts needed to introduce new language constructs (e.g., aggregates) are reduced.

Next to these theoretic benefits, justifications are also used in implementation of solvers.
In the unfounded set algorithm \cite{iclp/GebserKKS09}, justifications for atoms are stored (the so-called \emph{source-pointer approach} essentially maintains a justification).
\mycitet{lazygroundingASP} used justifications to learn new clauses to improve search in lazy grounding algorithms.
Additionally, justifications were used to improve parity game solvers \cite{vmcai/LapauwBD20}. 

The key semantic concept in justification theory is a \emph{justification}: an explanation why something holds (or does not hold) in a model.
Because of this, defining semantics of a logic through justification theory does not just provide benefits on the level of an analysis of its semantics, but also equips the logic with a mechanism of \emph{explanation}, thereby providing at least a partial answer to constantly increasing need 
for explainable methods in artificial intelligence \cite{ai/Miller19}, which is especially important in the light of the EU General Data Protection Regulation, article 22 of which  requires that all AI with an impact on human lives needs to be accountable.

Historically, justification theory was first defined in the doctoral thesis of \citet{DeneckerPhD93} as a framework for studying semantics of logic programs. 
In that work, a justification is a tree where the nodes are labelled with literals. 
Later, \citet{lpnmr/DeneckerBS15} developed a more general theory, aiming to also capture other knowledge representation formalisms. 
One notable difference with the early work was that justifications were no longer formalised as a tree, but as a graph. 
While it is known that for the major semantics of logic programming, these two concepts coincide, in the sense that they induce the same semantics, in general very little is known about the relation between tree-like and graph-like justifications; one of the goals of this paper is to bridge that gap.

One benefit of justification theory is that it provides a lot of freedom to create new semantics by means of so-called \emph{branch evaluations}.
However, as we show in this paper, not all branch evaluations lead to a well-defined semantics. 
Showing well-definedness can be an intricate job \cite{nmr/MarynissenPBD18} and until now no techniques have been developed to do this in a systematic way. 
We show that this issue closely relates to the coincidence of graph-like and tree-like justifications and provide relatively easy to verify criteria that guarantee that a finite justification framework is well-behaved. 
Our results build on existing work in the context of antagonistic games over graphs \cite{concur/GimbertZ05}: we establish a connection between justification theory and game theory, and exploit it to transfer existing results on antagonistic games over graphs to justification theory. 
The main contributions of this paper are as follows:

\begin{itemize}
 \item We present graph-based and tree-based justifications in a uniform way, thus enabling an in-depth study of their relationship. 
 \item We show that justifications induce games in the sense of \citet{concur/GimbertZ05}, thereby bridging a gap between non-monotonic reasoning and game theory. 
 \item Inspired by \citet{concur/GimbertZ05}, we develop two criteria for branch evaluations, namely monotonicity and selectivity, and we show for branch evaluations satisfying these conditions in \emph{finite} justification systems, graph-based justifications are well-behaved.
 While this means that our results do not apply, e.g., to infinitary semantics, for practical applications finiteness is not a major limitation.
 \item Additionally, we show that whenever graph-based justifications are well-behaved, then graph-based and tree-based justifications are equally powerful.
 \item Finally, we show that all branch evaluations corresponding to major semantics of logic programming are indeed selective and monotone.
\end{itemize}

As a consequence, with our results, proving that a branch evaluation induces well-behaved graph-based justifications is much easier: all that is needed is to show that it is monotone and selective.
This is in shrill contrast with the proofs of well-behavedness of  \citet{nmr/MarynissenPBD18} and \citet{DeneckerPhD93}, which span at least a couple of pages for each semantics.
The proof of Proposition \ref{prop:lpismonotoneandselective}, on the other hand, is much simpler and shorter.
\tplppaper{Additional proofs can be found at arXiv.}

\section{Justification systems}

%\bart{Algemene tip: Je schrijft heel vaak iets a la 
%
%``\citet{nmr/MarynissenPBD18} used four-valued logics'' 
%
%Ik zie veel liever staan:
%
%``\citet{nmr/MarynissenPBD18} used four-valued logics''
%}

In this section we introduce justification theory for three-valued logics.
\citet{nmr/MarynissenPBD18} used four-valued logics but the benefit of that still needs to be researched.
The formalisation we present here is the first in which  graph-like and tree-like justifications are unified in a single theory,
thereby facilitating a study of the relationship between them.
In the rest of this paper, let $\F$ be a set, referred to as a \emph{fact space},
such that $\lf = \setl{\Tr, \Fa, \Un} \subseteq \F$,
where $\Tr$, $\Fa$ and $\Un$ have the respective meaning \emph{true}, \emph{false}, and \emph{unknown}.
The elements of $\F$ are called \emph{facts}.
The set $\lf$ behaves as the three-valued logic with truth order $\Fa \leqt \Un \leqt \Tr$.
We assume that $\F$ is equipped with an involution $\tild: \F \rightarrow \F$ (i.e.\ a bijection that is its own inverse)
such that $\tild \Tr=\Fa$, $\tild \Un=\Un$, and $\tild x \neq x$ for all $x \neq \Un$.
For any fact $x$, $\tild x$ is called the \emph{complement} of $x$.
An example of a fact space is the set of literals over a propositional vocabulary $\Sigma$ extended with $\lf$ where $\tild$ maps a literal to its negation.
For any set $A$ we define $\tild A$ to be the set of elements of the form $\tild a$ for $a \in A$.
We distinguish two types of facts: \emph{defined} and \emph{open} facts.
The former are accompanied by a set of rules that determine their truth value.
The truth value of the latter is not governed by the rule system but comes from an external source or is fixed (as is the case for logical facts).

\begin{definition}
  A \emph{justification frame} $\jf$ is a tuple $\jfcomplete$ such that
  \begin{itemize}
    \item $\Fd$ is a  subset of $\F$ closed under $\tild$, i.e.\ $\tild \Fd = \Fd$; facts in $\Fd$ are called \emph{defined};
    \item no logical fact is defined: $\lf \cap \Fd = \emptyset$;
    \item $\rules \subseteq \Fd \times 2^\F$;
    \item for each $x \in \Fd$, $(x, \emptyset) \notin R$ and there is an element $(x, A) \in R$ for $\emptyset \neq A \subseteq \F$.
  \end{itemize}
\end{definition}
The set of \emph{open} facts is denoted as $\Fo:=\F\setminus\Fd$.
An element $(x, A) \in \rules$ is called a \emph{rule} with \emph{head} $x$ and \emph{body} (or \emph{case}) $A$.
The set of cases of $x$ in $\jf$ is denoted as $\jf(x)$.
Rules $(x, A) \in \rules$ are denoted as $x \gets A$ and if $A=\setl{y_1, \ldots, y_n}$, we often write $x \gets y_1, \ldots, y_n$.
% Cases are allowed to be infinite.

Logic programming rules can easily be transfered to rules in a justification frame.
However, in logic programming, only rules for positive facts are given; never for negative facts.
Hence, in order to apply justification theory to logic programming, a mechanism for deriving rules for negative literals is needed as well.
For this, a technique called \emph{complementation} was invented; it is a generic mechanism that allows turning a set of rules for $x$ into a set of rules for $\tild x$.

Before defining complementation, we first define \emph{selection functions} for $x$.
A selection function for $x$ is a mapping $s\colon \jf(x) \rightarrow \F$ such that $s(A) \in A$ for all rules of the form $x \gets A$.
Intuitively, a selection function chooses an element from the body of each rule of $x$.
% The existence of selection functions in general is guaranteed by the axiom of choice.
For a selection function $s$, the set $\{s(A) \mid A \in \jf(x)\}$ is denoted by $\im(s)$.
\begin{definition}
  Given a set of rules $\rules$.
  We define $\rules^*$ to be the set of elements of the form $\tild x \gets \tild Im(s)$
  for $x \in \Fd$ that has rules in $\rules$ and $s$ a selection function for $x$.
  The \emph{complementation} of $\jf$ is defined as $\jfcomplete[\rules \cup \rules^*]$.
  A justification frame $\jf$ is \emph{complementary} if it is fixed under complementation.
\end{definition}
In case of logic programming where only rules for positive facts are given, we can construct a complementary justification frame that is ``equivalent'' by first taking the complementation and then adding rules of the form $x \gets A'$ if there is a rule $x \gets A \in \rules$ with $A \subseteq A' \subseteq \F$.
If a logic program is finite, then so is this complementary justification frame.
Remark that complementarity is a natural property that is satisfied in all practical applications of justification theory.
This is because in all practical applications you start with rules for only of the polarities, and apply complementation to this set of rules.
\begin{definition}
  A \emph{directed labelled graph} is a quadruple $(N, L, \edges, \ell)$ where $N$ is a set of nodes, $L$ is a set of labels, $\edges \subseteq N \times N$ is the set of edges, and $\ell\colon N \rightarrow L$ is a function called the labelling.
  An \emph{internal} node is a node without outgoing edges.
\end{definition}

\begin{definition}
  Let $\jf=\jfcomplete$ be a justification frame.
  A \emph{graph-like justification} $J$ in $\jf$ is a directed labelled graph $(N, \Fd, \edges, \ell)$ such that
  \begin{itemize}
    \item for every internal node $n \in N$ it holds that $\ell(n) \gets \{\ell(m) \mid (n,m) \in \edges\} \in \rules$;
    \item no two nodes have the same label,~i.e., $\ell$ is injective.
  \end{itemize}
  A \emph{tree-like justification} $J$ in $\jf$ is a directed labelled graph $(N, \Fd, \edges, \ell)$ such that
  \begin{itemize}
    \item the underlying undirected graph is a forest,~i.e., is acyclic;
    \item for every internal node $n \in N$ it holds that $\ell(n) \gets \{\ell(m) \mid (n,m) \in \edges\} \in \rules$.
  \end{itemize}
\end{definition}
We write $\justifications_g(x)$ (resp.~$\justifications_t(x)$) to denote the set of graph-like (resp.~tree-like) justifications that have a node labelled $x$. Whenever we say ``justification'' it can be either a graph-like or a tree-like.

\begin{definition}
  A justification is \emph{locally complete} if it has no leaves with label in $\Fd$, and it is \emph{connected} if the underlying undirected graph is connected.
  We call $x \in \Fd$ a \emph{root} of a justification $J$ if there is a node $n$ labelled $x$ such that every node is reachable from $n$ in $J$.
\end{definition}
\begin{definition}
  Let $\jf$ be a justification frame.
  A $\jf$-\emph{branch} is either an infinite sequence in $\Fd$ or a finite non-empty sequence in $\Fd$ followed by an element in $\Fo$.
  For a justification $J$ in $\jf$, a $J$-\emph{branch} starting in $x \in \Fd$ is a path in $J$ starting in $x$ that is either infinite or ends in a leaf of $J$.
  We write $\branches_J(x)$ to denote the set of $J$-branches starting in $x$. 
\end{definition}
Not all $J$-branches are $\jf$-branches  since they can end in nodes with a defined fact as label.
However, if $J$ is locally complete, any $J$-branch is also a $\jf$-branch.

We denote a branch $\branch$ as $\branch:x_0 \rightarrow x_1 \rightarrow \cdots$ and define $\tild \branch$ as $\tild x_0 \rightarrow \tild x_1 \rightarrow \cdots$.

\begin{definition}
  A \emph{branch evaluation} $\be$ is a mapping that maps any $\jf$-branch to an element in $\F$ for all justification frames $\jf$.  A branch evaluations $\be$ is \emph{consistent} if $\be(\tild \branch) = \tild \be(\branch)$ for any branch $\branch$.
  A justification frame $\jf$ together with a branch evaluation $\be$ form a \emph{justification system} $\js$, which is presented as a quadruple $\jscomplete$. 
% \begin{definition}
% \end{definition}
\end{definition}
We now define some branch evaluations that induce semantics (as shown later) corresponding to the equally named semantics of logic programs.
% We start with some particular branch evaluations.

\begin{definition}
  The \emph{supported} (completion) branch evaluation $\besp$ maps $x_0 \rightarrow x_1 \rightarrow \cdots$ to $x_1$.
  The \emph{Kripke-Kleene} branch evaluation $\bekk$ maps finite branches to their last element and infinite branches to $\Un$.
\end{definition}
This definition is similar to the one of \citet{lpnmr/DeneckerBS15}, except that they define $\besp$ to map finite branches to their last element.
Our definition is closer to the intuitions behind completion semantics, which compares the value of rule bodies with the value of its head. 
% Our definition is more intuitive since completion semantics looks at rule bodies as an explanation.
However, the induced semantics are the same and it is not difficult to see that all good properties of $\besp$ we prove (monotonicity and selectivity) also hold for the definition of \citet{lpnmr/DeneckerBS15}.
% We know define three-valued interpretations ($\lf = \setl{\Tr, \Un, \Fa}$)

\begin{definition}
  A \emph{(three-valued) interpretation} of $\F$ is a function $\interp:\F \rightarrow \lf$ such that $\interp(\tild x) = \tild \interp(x)$ and $\interp(\ell) = \ell$ for all $\ell \in \lf$.
  The set of interpretations of $\F$ is denoted by $\mathfrak{I}_\F$.
\end{definition}

\begin{definition}
  Let $\js=\jscomplete$ be a justification system, $\interp$ an interpretation of $\F$, and $J$ a locally complete graph-like (respectively tree-like) justification in $\js$.
  Let $x \in \Fd$ be a label of a node in $J$.
  The \emph{value} of $x \in \Fd$ by $J$ under $I$ is defined as $\jval(J, x, \interp) = \bigwedge_{\branch \in \branches_J(x)} \interp(\be(\branch))$,
  where $\bigwedge$ is the greatest lower bound with respect to $\leqt$.
  
  The \emph{graph-like supported value} of $x \in \F$ in $\js$ under $\interp$ is defined as
  \begin{equation*}
    \suppvalue_g(x, \interp) = \bigvee_{J \in \justifications_g(x)} \jval(J,x,\interp) \text{  for $x \in \Fd$  \qquad $\suppvalue_g(x, \interp)=\interp(x)$ for $x \in \Fo$.
}
  \end{equation*}
  Likewise, the \emph{tree-like supported value} of $x \in \F$ in $\js$ under $\interp$ is defined as
  \begin{equation*}
    \suppvalue_t(x, \interp) = \bigvee_{J \in \justifications_t(x)} \jval(J,x,\interp)\text{  for $x \in \Fd$  \qquad $\suppvalue_g(x, \interp)=\interp(x)$ for $x \in \Fo$.
}
  \end{equation*}
%   if $x \in \Fd$ and $\suppvalue_t(x, \interp)=\interp(x)$ if $x \in \Fo$.
\end{definition}
 
The following proposition shows that for determining the supported value of $x$ it suffices to look at justifications that are connected and that have $x$ as a root.
This is also the reason why tree-like justifications are not called forest-like justifications.

\thmwithproof{prop:supportonlyrootedjustification}{Proposition}{proposition}{%
  For any locally complete graph-like (respectively tree-like) justification $J$ and $x\in \Fd$ a label of an internal node,
  there is a connected and locally complete graph-like (respectively tree-like) justification $J'$ such that $x$ is a root of $J'$ and $\jval(J,x,\interp) \leqt \jval(J',x,\interp)$ for all interpretations $\interp$.
}{
  Take a node $n$ labelled $x$.
  Let $J'$ be the subgraph of $J$ spanned by the nodes that are reachable from $n$.
  By definition $J'$ is connected.
  Since $J$ is localy-complete, $J$ will have no leafs with a defined fact as label.
  By construction $\branches_{J'}(x) \subseteq \branches_J(x)$, which means that $\jval(J,x,\interp) \leqt \jval(J',x,\interp)$ for all interpretations $\interp$.
}
%\simon{I don't think we will actually NEED}
%We extend our definition of supported value by a justification to non-locally complete justifications.
%
%\begin{definition}
%  A justification $K$ is an \emph{extension} of a justification $J$ if $J$ is a subgraph of $K$.
%\end{definition}
%\begin{definition}
%  The graph-like (respectively tree-like) supported value of $x \in \Fd$ by a (not necessarily locally complete) justification $J$ under $\interp$ is defined as the least upper bound of the supported values of $x$ by graph-like (respectively tree-like) locally complete extensions of $J$.
%\end{definition}
%
%This definition of support is compatible with the previous one since if $J$ is a locally complete justification and $K$ a locally complete extension of $J$,
%then any $K$-branch starting in $x$ that is a label of a node in $J$ is also a $J$-branch;
%hence the support value of $x$ under $K$ is equal to the support value of $x$ under $J$.
%This alternative definition is useful in proofs, but also allows to give justifications that do not contain unnecessary information.
%For instance, under completion semantics, we do not need locally complete justifications, but just a single rule.\bart{Unclear}
In many logical frameworks, there is an asymmetry between positive and negative literals.
Consider for instance stable semantics of logic programs \cite{iclp/GelfondL88}.
There, the default value for atoms is false; as such, the default value for its negation is true.
Thus, the set of defined literals is divided into two parts, those that are default and those that are not (sometimes called \emph{deviant}).%\footnote{In some contexts, non-default literals are called .}
% In justification theory, such a distinction is not present, but can be introduced by means of a sign function.
This idea is generalised to justification theory in signed justification frames.
\begin{definition}
  Let $\jf$ be a justification frame.
  A \emph{sign function} on $\jf$ is a map $\sgn: \Fd \rightarrow \setl{{-}, {+}}$ such that $\sgn(x) \neq \sgn(\tild x)$ for all $x \in \Fd$.
  We denote $\Fn\coloneqq \sgn^{-1}(\setl{{-}})$ and $\Fp\coloneqq \sgn^{-1}(\setl{{+}})$.
%   Remark that $\setl{\Fn, \Fp}$ is a partition of $\Fd$.
\end{definition}
From now on, we fix a sign function on $\jf$. We say that an infinite branch has a positive (respectively negative) tail if from some point onwards all elements are in $\Fp$ (respectively $\Fn$).

\begin{definition}
  The \emph{well-founded} branch evaluation $\bewf$ maps finite branches to their last element. It maps infinite branches to $\Tr$ if they have a negative tail, to $\Fa$ if they have a positive tail and to $\Un$ otherwise.
  
  The \emph{stable} (answer set) branch evaluation $\best$ maps a branch $x_0 \rightarrow x_1 \rightarrow \cdots$ to the first element that has a different sign than $x_0$ if it exists; otherwise $\branch$ is mapped to $\bewf(\branch)$.
\end{definition}
% \begin{remark}
In the next section we discuss how justification systems induce semantics. 
For the case of logic programs, the semantics induced by the well-founded branch evaluation coincides  with the well-founded semantics \cite{GelderRS91} and the one induced by the stable branch evaluation coincides with the stable semantics \cite{iclp/GelfondL88}.
%   How semantics are induced by a branch evaluation is discussed in the next section.
% \end{remark}

\section{Research questions}

% \bart{old motivation from the NMR paper.}

A semantics using justification theory is defined as follows:
an interpretation $\interp$ is a ``model'' according a justification system if the supported value of each fact is equal to its value in $\interp$, i.e., if $\suppvalue_g(x, \interp) = \interp(x)$ for each $x\in \Fd$. 
Given such a model, a justification with the best (largest with respect to $\leqt$) value of $x$ constitutes an explanation why $x$ has its particular value. 

In many logics, including logic programming, there is an asymmetry between the defined facts by focusing either only on the positive or only on the negative facts. %, as discussed above. % in the context of complementation of a justification frame. 
At the level of the justification frame, this asymmetry is resolved by \emph{complementation}. 
At the semantic level,  such logics also only focus on the value assigned to facts of certain polarity (most often, positive facts, i.e., atoms). 
This induces the following question: ``a justification for $x$ with value $\Tr$ can be seen as an explanation of why $x$ is true; but which semantic structure can explain why $x$ is false?''
From the definition of supported value, it can be seen that $x$ is false if \emph{there are no justifications} with a better value for $x$.
The question that then remains is: how to show that there are no better justifications for $x$. The most obvious solution is considering a justification of $\tild x$. Indeed: intuitively, an explanation why the negation of $x$ is true should explain why $x$ is false.
However, this method implicitly assumes that $\suppvalue(\tild x,\interp) = \tild \suppvalue(x,\interp)$,
a property which was called the \emph{consistency property} by \citet{nmr/MarynissenPBD18}.
Consistency is a reasonable assumption that, unfortunately, is not always satisfied, not even for complementary justification systems.

For the most common branch evaluations in logic programming (completion, Kripke-Kleene, well-founded, and stable), consistency has been shown in the context of graph-like justifications by \citeNS{nmr/MarynissenPBD18} and in the context of tree-like justifications by \citeNS{DeneckerPhD93}. However, these proofs are often very intricate, consist of many pages and there are currently no insights about \emph{under which conditions on a branch evaluation} consistency is guaranteed. 
% This constitutes our first research question. 

\smallskip

\noindent \fbox{\parbox{\dimexpr\linewidth-2\fboxsep-2\fboxrule\relax}{ 
% \begin{quote}
 \textbf{Research question 1:} under which conditions on branch evaluations is the consistency property satisfied (i.e., is $\suppvalue(\tild x, \interp) = \tild \suppvalue(x,\interp)$ for each $x$ and $\interp$)? 
% \end{quote}
}}\smallskip

%\citet{nmr/MarynissenPBD18} proved that $\suppvalue_g(\tild x, \interp) \leqt \tild \suppvalue_g(x, \interp)$ holds for consistent branch evaluations.
%For an alternative proof see the proof of Proposition \ref{prop:consistentbepartialconsistentsupportvalue}.
%\bart{Zou voorgaande zin weg mogen? Op dit punt in de paper is het onderscheid tussen consistentency van branch evaluations en ``de consistency property ''  niet duidelijk.}

\ignore{
Why do we want this?
Most logic programming semantics only work with rules for atomic facts (or only the negation of atomic facts).
In those logics, we can provide an explanation why an atomic fact
is true in an interpretation but we cannot provide an explanation why it is
false.
It is false because there is no explanation that it is true.
On the other hand, justification semantics has explanations for both atomic facts and their negation.
Therefore, it is a natural property that these explanations are opposite.
This is expressed exactly by $\suppvalue(\tild x, \interp) = \tild \suppvalue(x,\interp)$, which we call the \emph{consistency} or \emph{well-definedness} of justification semantics.

In \citeNS{nmr/MarynissenPBD18}, the authors proved that this is the case for completion, Kripke-Kleene, well-founded, and stable branch evaluations for graph-like justifications.
In \cite{DeneckerPhD93}, the same result is proved for tree-like justifications.
}

Next to guaranteeing the ability to explain falsity of atoms similar to truth, consistency brings another advantage: in case of consistency of graph-like (respectively tree-like) justifications, a justification system induces an operator (the support operator)
\begin{align*}
 \supp^g:&\mathfrak{I}_\F \rightarrow \mathfrak{I}_\F: \interp \mapsto \suppvalue_g(\cdot, \interp), \text{ respectively}\\
 \supp^t:&\mathfrak{I}_\F \rightarrow \mathfrak{I}_\F: \interp \mapsto \suppvalue_t(\cdot, \interp).
\end{align*}
In case of non-consistency, the function $\suppvalue_g(\cdot, \interp)$ does not necessarily define an interpretation. 
Having such an operator defined allows bridging a gap towards operator-based studies of semantics, such as done, e.g., in  approximation fixpoint theory \mycite{AFT}.

%\simon{Review: ``Do you want to state the conditions for a branch evaluation to have the consistency property by using the connection with the induced semantics or the other way around? Are you interested in a simpler way to prove the results in (Marynissen et al. 2018), where consistency for the most common branch evaluations in logic programming was proven? Probably I've misunderstood something since you are assuming consistency to take advantage of support operators.''}
%\bart{Ik weet niet of we hier in de tekst nog iets aan moeten doen? 
%In de cover letter zou ik hierop antwoorden dat we ``principled and simple methods willen om consistency te bewijzen, en dus onder andere de resultaten van marynissen op een eenvoudigere manier willen terugvinden. We hopen dat het met de verduidelijkingen in de tekst  beter naar voor komt}

\ignore{
This can be made more rigorously as follows.

For any interpretation $\interp$ of $\F$ we have the mappings $\suppvalue_g(\cdot, \interp):\F \rightarrow \lf$ and $\suppvalue_t(\cdot, \interp):\F \rightarrow \lf$.
In order for $\interp$ to be a model, we should have that $\suppvalue(\cdot,\interp) = \interp$.

Define the \emph{graph-like support operator}
\begin{equation*}
\supp^g:\mathfrak{I}_\F \rightarrow (\F \rightarrow \lf): \interp \mapsto \suppvalue_g(\cdot, \interp).
\end{equation*}
and the \emph{tree-like support operator}
\begin{equation*}
\supp^t:\mathfrak{I}_\F \rightarrow (\F \rightarrow \lf): \interp \mapsto \suppvalue_t(\cdot, \interp).
\end{equation*}

Most of the time, we want that the ranges of $\supp^g$ and $\supp^t$ are actually $\mathfrak{I}_\F$.
So that each interpretation is actually mapped to an interpretation.
Therefore, $\suppvalue(\cdot,\interp)$ should be an interpretation,~i.e., $\suppvalue(\tild x,\interp) = \tild \suppvalue(x,\interp)$.
}

We now define a branch evaluation that illustrates that consistency is not always guaranteed. 
% The following provides an example of a branch evaluation that will serve as a counter example for inconsistent justification semantics and tree-like justifications being more powerful than graph-like justifications.

\begin{definition}\label{def:ex}
  For a signed justification frame $\jf = \jfcomplete$, we define the branch evaluation $\besub{ex}$ as follows:
  \begin{itemize}
    \item $\besub{ex}(x_0 \rightarrow x_1 \rightarrow \cdots\rightarrow x_n)=x_n$;
    \item $\besub{ex}(x_0 \rightarrow x_1 \rightarrow \cdots)=\Fa$ if $\exists i_0 \in \mathbb{N}:\forall i,j > i_0: x_i \in \Fp$ and if $x_i=x_j$, then $x_{i+1}=x_{j+1}$;
    \item $\besub{ex}(x_0 \rightarrow x_1 \rightarrow \cdots)=\Tr$ if $\exists i_0 \in \mathbb{N}:\forall i,j > i_0: x_i \in \Fn$ and if $x_i=x_j$, then $x_{i+1}=x_{j+1}$; 
    \item $\besub{ex}(x_0 \rightarrow x_1 \rightarrow \cdots)=\Un$ otherwise.
  \end{itemize}
%   Finite branches are mapped to their last element.
\end{definition}
\begin{example}\label{ex:niko}
  Take $\Fd=\setl{a,\tild a,b,\tild b,c,\tild c}$ and $\Fo=\lf$, $\Fp=\setl{a,b,c}$, and $\Fn=\setl{\tild a, \tild b, \tild c}$.
  The set of rules is the complementation of $\setl{a\gets b;a\gets c;b\gets a;c\gets a}$.
  Under the branch evaluation $\besub{ex}$, we have that $\suppvalue_g(a, \interp) = \Fa$, while $\suppvalue_g(\tild a, \interp) = \Un$ for any interpretation $\interp$ of $\F$, i.e., the consistency property is not satisfied here.
  This can easily be checked by noting that the only connected graph-like justifications with $a$ or $\tild a$ as root are the following
  \[
%   \begin{center}
    \begin{tikzpicture}[transform shape]
    \node (a) at (0, 0) {$a$};
    \node (b) at (1, 0) {$b$};
    \tikzstyle{EdgeStyle}=[bend left, style={->}]
    \Edge(a)(b)
    \Edge(b)(a)
    \end{tikzpicture}
    \hskip 5em
    \begin{tikzpicture}[transform shape]
    \node (a) at (0, 0) {$a$};
    \node (c) at (1, 0) {$c$};
    \tikzstyle{EdgeStyle}=[bend left, style={->}]
    \Edge(a)(c)
    \Edge(c)(a)
    \end{tikzpicture}
    \hskip 5em
    \begin{tikzpicture}[transform shape]
    \node (a) at (1.3, 0) {$\tild a$};
    \node (b) at (0, 0) {$\tild b$};
    \node (c) at (2.6, 0) {$\tild c$};
    \tikzstyle{EdgeStyle}=[bend left, style={->}]
    \Edge(a)(b)
    \Edge(a)(c)
    \Edge(b)(a)
    \Edge(c)(a)
    \end{tikzpicture}
%   \end{center}
\]
  Additionally, it can also be seen that $\suppvalue_t(a, \interp) = \Un$ and $\suppvalue_t(\tild a, \interp) = \Un$ for each interpretation $\interp$,
  and thus that the graph-like supported value can differ from the tree-like supported value.
  A tree-like justification $J$ with $\jval(J,a,\interp) = \Un$ is given below:
  \[
%   \begin{center}
    \begin{tikzpicture}[transform shape]
    \node (a) at (0, 0) {$a$};
    \node (b) at (1, 0) {$b$};
    \node (a2) at (2, 0) {$a$};
    \node (c) at (3, 0) {$c$};
    \node (a3) at (4, 0) {$a$};
    \node (b2) at (5, 0) {$b$};
    \node (a4) at (6, 0) {$a$};
    \node (c2) at (7,0) {$c$};
    \node (d) at (8,0) {$\cdots$};
    
    \tikzstyle{EdgeStyle}=[style={->}]
    \Edge(a)(b)
    \Edge(b)(a2)
    \Edge(a2)(c)
    \Edge(c)(a3)
    \Edge(a3)(b2)
    \Edge(b2)(a4)
    \Edge(a4)(c2)
    \Edge(c2)(d)
    \end{tikzpicture}
%   \end{center}
\]
  A tree-like justification $J$ with $\jval(J,\tild a, \interp) = \Un$ is given by the tree unfolding of the graph-like justification for $\tild a$ with value $\Un$.
\end{example}
This example immediately brings us to the second research question, namely the relation between graph-like and tree-like justifications. 

\smallskip

\noindent \fbox{\parbox{\dimexpr\linewidth-2\fboxsep-2\fboxrule\relax}{ 
% \begin{quote}
 \textbf{Research question 2:} under which conditions on branch evaluations do the graph-like and the tree-like supported value coincide (i.e., is $\suppvalue_t(x, \interp) = \suppvalue_g(x,\interp)$ for each $x$ and $\interp$)? 
% \end{quote}
}}\smallskip
\ignore{
This illustrates that whenever you define a new branch evaluation, you have to show that the induced semantics are well-defined.
In this paper, we give sufficient conditions that imply the consistency of the justification semantics.

The example also provides an instance where tree-like justifications are more powerful than graph-like justifications, since $\suppvalue_g(a,\interp) <_t \suppvalue_t(a,\interp)$.
This brings us to the second issue we discuss: in what case will graph-like justifications render the same results as tree-like justifications.
}
Resolving this question is important from a practical perspective, e.g., for justification-based algorithms.
Indeed, tree-like justifications  tend to be infinitely large when recursion is present.
Graph-like justifications on the other hand, can easily be maintained. 
It is easy to see that tree-like justifications are at least as powerful as graph-like justifications.
\thmwithproof{prop:treebetterthangraph}{Proposition}{proposition}{%
  For any $x \in \Fd$ and interpretation $\interp$, it holds that $\suppvalue_g(x, \interp) \leqt \suppvalue_t(x, \interp)$.
}{
  Any graph-like justification $J$ can be unrolled into a tree-like justification $J'$ with the same branches.
  Therefore, $\jval(J,x,\interp) = \jval(J',x,\interp)$, which proves the result.
}
The fact that the branch evaluation from Definition \ref{def:ex} serves as a counterexample for both research questions is not a coincidence: these questions are closely related. 
% The two issues discussed above are not independent.
We will show (Theorem \ref{th:consistentgraphgivesequaltrees}) that under very unrestrictive conditions, if graph-like justifications are consistent, then tree-like justifications give the same result as graph-like justifications, and thus are also consistent.

As a technical means to answer these questions, we use game theory: we will show that justification systems induce a game in the sense of \cite{concur/GimbertZ05} and that coincidence of graph-like and tree-like justifications is inherently tied to the existence of optimal positional strategies in the induced game. 
Before doing so, we introduce some preliminaries on games.

\section{Games}
We now define antagonistic games over graphs \cite{concur/GimbertZ05}. Our formalisation slightly differs from the one of \citeANP{concur/GimbertZ05}, but the differences are minor. % and non-fundamental. 
% Some details are slightly different, but they are no fundamental.
% 
Intuitively a game  unfolds as follows:
There are two players $\ptr$ and $\pfa$, with opposite interests.
Let $\graph$ be a graph in which each vertex is owned either by $\ptr$, or by $\pfa$.
Initially, a stone is placed on some vertex in $\graph$.
At each step, the player owning the vertex with the stone moves the stone along an outgoing edge.
The players interact in this way ad infinitum or until the stone is on a vertex without outgoing edges.
% This means a play in the game is either an infinite path or a path ending in a leaf of the graph.
% More formally, this becomes:

\begin{definition}
  A \emph{game graph} is a quadruple $\graph=(\states, \states_\ptr, \states_\pfa, \edges)$
  where $\states_\ptr$ and $\states_\pfa$ form a partition of the set of states $\states$, and $\edges \subseteq \states \times \states$.
  For a transition $e = (s,t) \in \edges$, the states $s$ and $t$ are respectively called the source and target of $e$ (denoted $\source(e) $ and $\target(e)$).
  For a state $s \in \states$, $s\edges$ is the set of outgoing edges from $s$.
  A \emph{path} in $\graph$ is a non-empty finite or infinite sequence of states $p=s_0s_1s_2\ldots$ such that for all $i\geq 0$, $(s_i,s_{i+1}) \in E$.
  We define $\source(p)$ to be equal to $s_0$.
  If $p$ is finite, then $\target(p)$ is the last state of $p$.
  If a path consists of a single state $s$, we call it \emph{empty}. %then we say that it is empty.
  We denote the empty path in $s$ by $\lambda_s$.
  The set of finite paths in $\graph$ (including the empty paths) is denoted by $\paths{\graph}$.
\end{definition}
% 
% Plays indicate possibly infinite paths that indicate a complete play of a game.
\begin{definition}
  Let $\graph$ be a game graph.
  A \emph{play} in $\graph$ is either an infinite path in $\graph$ or a finite path in $\graph$ that ends in a state without outgoing edges.
  The set of plays in $\graph$ is denoted by $\plays_\graph$.
\end{definition}
In its most general form, games of this form do not have a winner, but instead, each player has a preference relation indicating which plays they prefer over others.

% Players prefer some plays over other plays.
% This is formalised with a preference relation.
\begin{definition}
  Let $\graph$ be a game graph.
  A \emph{preference relation} for a player $P$ is a total preorder (i.e., a reflexive and transitive relation such that for all plays $p$ and $q$, $p \pref_P q$ or $q \pref_P p$ holds) over $\plays_\graph$.
% \end{definition}
% 
% \begin{remark}
%   Here, total means that for all $p, q \in \plays_\graph$, $p \pref_P q$ or $q \pref_P p$ holds.
% \end{remark}

\noindent
% \begin{definition}
  A \emph{game} is a tuple $\game=(\graph, \preftr)$, where $\graph$ is a game graph and $\preftr$ is a preference relation for $\ptr$.
  We define $\preffa$ to be the inverse relation of $\preftr$.
\end{definition}
Intuitively, if $p \pref_P q$, then the player $P$ prefers the play $q$ at least as much as $p$.
If both $p \pref_P q$ and $q \pref_P p$ hold, then we say that $p$ and $q$ are equivalent with respect to $\pref_P$.
% We here study antagonistic games, i.e., games in which the players have opposite interests.

When playing, players usually follow some strategy. In this paper we are interested in two types of strategies: \emph{general} and \emph{positional} strategies. Intuitively, the former can take into account the entire play history to determine a move, while the latter only depends on the current state. 
% A general strategy is one that takes into account every state it encountered while playing.

\begin{definition}
  Let $\graph = (\states, \states_\ptr, \states_\pfa, \edges)$ be a game graph.
  A \emph{general strategy} for a player $P$ in $\graph$ is a function $\sigma_P: \setprop{p \in \paths{\graph}}{ \target(p) \in \states_P} \rightarrow \edges$ such that $\sigma_P(p) \in \target(p)\edges$.
  The set of general strategies for player $P$ is denoted by $\gstrategies(P)$.
% \end{definition}
% % 
% % In contrast, positional strategies do not depend on the history, and thus only depend on which state the game has stranded.
% % 
% \begin{definition}
  A \emph{positional (or memoryless) strategy} for a player $P$ in $\graph$ is a mapping $\sigma_P: \states_P \rightarrow \edges$ such that $\sigma_P(s) \in s\edges$ for all $s \in \states_P$.
  The set of positional strategies for $P$ is denoted by $\pstrategies(P)$.
\end{definition}
We can embed $\pstrategies(P)$ into $\gstrategies(P)$ by mapping a positional strategy $\sigma$ to a general strategy that maps $s_0\ldots s_n$ to $\sigma(s_n)$.
Slightly abusing notation, we thus have that $\pstrategies(P) \subseteq \gstrategies(P)$.

\begin{definition}
  Let $\graph = (\states, \states_\ptr, \states_\pfa, \edges)$ be a game graph.
  A finite or infinite path $p=s_0s_1s_2\ldots$ in $\graph$ is \emph{consistent} with a general strategy $\sigma_P$  for a player $P$ if
  $\sigma_P(s_0\ldots s_i) = (s_i,s_{i+1})$ whenever $s_i \in \states_P$.
  It is \emph{consistent} with a positional strategy $\sigma_P$ for player $P$ if
  %for all $i \geq 0$ it holds that if 
   $\sigma_P(s_i) = (s_i,s_{i+1})$ whenever $s_i \in \states_P$.
\end{definition}
Given a state $s$ and strategies $\sigma$ and $\tau$ for players $\ptr$ and $\pfa$, there exists a unique play in $\graph$ starting in $s$ consistent with both $\sigma$ and $\tau$.
We denote this play by $\play{\graph}(s, \sigma, \tau)$.

\begin{definition}
  Let $\sigma$ be a general strategy for a player $P$.
  The \emph{play tree} for $\sigma$ in $x$ is the tree with nodes labelled by states whose maximally long branches correspond to the plays starting in $x$ consistent with $\sigma$. %We mean that if $yz$ is part of such a play, then $(y, z)$ is an edge in this play tree.
  \\
% \end{definition}
% 
% \begin{definition}
  Let $\sigma_p$ be a positional strategy for a player $P$.
  The \emph{play graph} for $\sigma_p$ in $x$ is the subgraph of the game graph consisting of all edges occurring in plays starting in $x$ consistent with $\sigma_p$.
%   An edge $(y,z)$ is in a play, if $yz$ is part of the sequence of states of the play.
\end{definition}

Under the assumption that $\ptr$ and $\pfa$ are rational agents, trying to maximize the value of the resulting play in their preference relation ($\preftr$ respectively $\preffa$), some strategies will not be followed. 
So-called \emph{Nash equilibria} are often used as ``solutions'' to games;
intuitively, they are pairs of strategies from which neither player wants to unilaterally deviate.
In antagonistic games over graphs this is defined as follows.

\begin{definition}
  Let $\graph = (\states, \states_\ptr, \states_\pfa, \edges)$ be a game graph and let $\sigma^*$ and $\tau^*$ be general strategies for respectively $\ptr$ and $\pfa$.
  The pair $(\sigma^*, \tau^*)$ is \emph{optimal} if for all states $s \in \states$ and all general strategies $\sigma$ and $\tau$ for respectively $\ptr$ and $\pfa$ it holds that
  $\play{\graph}(s, \sigma, \tau^*) \preftr \play{\graph}(s, \sigma^*, \tau^*) \preftr \play{\graph}(s, \sigma^*, \tau)$.
\end{definition}
Intuitively, in an optimal pair of strategies, no player can benefit by changing only their strategy.
In this paper, we will be particularly interested in the question whether an optimal pair of \emph{positional} strategies (i.e., an optimal pair of strategies such that $\sigma^*$ and
$\tau^*$ are positional) exists.
% Therefore, the strategies are in some equilibrium state.

\section{Embedding justification theory in game theory}

% In the previous section, antagonistic games over graphs are introduced.
We now show how to derive an antagonistic game from a justification system. 
We first introduce a construction and then show how strategies in the constructed game correspond to justifications of the original system, which will serve as the basis for proving the consistency property.
% In this section, we will relate the concepts of these games with the ones of justification theory.
The following serves as a running example for this section.
\begin{example}\label{ex:leading}
  Let $\jf = \jfcomplete$ with $\Fd = \setl{p,q,\tild p, \tild q}$, $\Fo=\lf \cup \setl{r,\tild r}$, and $\rules$ the complementation of $\setl{p \gets \tild q; q \gets \tild p; p \gets r}$.
\end{example}

\paragraph{Games associated to justifications}
Let $\jf$ be a justification frame.
For any $x \in \Fd$ and rule $x \gets A \in \rules$, we introduce new symbols $r_{x \gets A}$, which we call rule symbols.
In the game associated with a justification system, $\ptr$ owns the defined facts and $\pfa$ owns the rule symbols.

\begin{definition}
  The game graph $\graph_\jf=(\states, \states_\ptr, \states_\pfa, \edges)$ corresponding to the justification frame $\jf$ is defined by $\states_\ptr = \F$, $\states_\pfa = \setprop{r_{x\gets A}}{x \in \Fd, x \gets A \in \rules}$, $\states = \states_\ptr \cup \states_\pfa$, and for any $x \in \Fd$, $x \gets A \in \rules$, and $y \in A$ we have edges $(x, r_{x \gets A})$ and $(r_{x \gets A},y)$.
\end{definition}
\begin{example}
  For the justification frame from Example \ref{ex:leading}, the game graph (without isolated nodes) is visualized below.
   
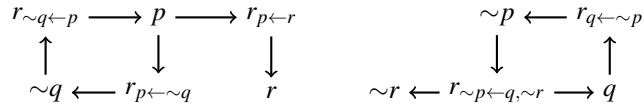
\[
%   \begin{center}
    \begin{tikzpicture}[transform shape]
    \node (p) at (1.5, 1) {$p$};
    \node (rp) at (1.5, 0) {$r_{p \gets \tild q}$};
    \node (rp2) at (3, 1) {$r_{p \gets r}$};
    \node (r) at (3, 0) {$r$};
    \node (nq) at (0, 0) {$\tild q$};
    \node (rnq) at (0, 1) {$r_{\tild q \gets p}$};
    
    \node (np) at (6, 1) {$\tild p$};
    \node (rnp) at (6, 0) {$r_{\tild p \gets q, \tild r}$};
    \node (nr) at (4.5, 0) {$\tild r$};
    \node (q) at (7.5, 0) {$q$};
    \node (rq) at (7.5, 1) {$r_{q \gets \tild p}$};
    
    \tikzstyle{EdgeStyle}=[style={->}]
    \Edge(p)(rp)
    \Edge(p)(rp2)
    \Edge(q)(rq)
    \Edge(np)(rnp)
    \Edge(nq)(rnq)
    \Edge(rp)(nq)
    \Edge(rp2)(r)
    \Edge(rnp)(q)
    \Edge(rnp)(nr)
    \Edge(rnq)(p)
    \Edge(rq)(np)
    \end{tikzpicture}
%   \end{center}
\]
\end{example}
In a game graph corresponding to a justification frame, a strategy for $\ptr$ chooses a rule for every defined fact and a strategy for $\pfa$ chooses for each rule an element from its body. Intuitively, this means that the true player will be the one choosing a justification, while the false player chooses a branch in that justification.

Any play in this graph corresponds to a $\jf$-branch as follows:
Let $s_0s_1\ldots$ be a play in $\graph_\jf$.
By removing the rule symbols, we get a sequence in $\F$, which is a $\jf$-branch.
Therefore, any play $p \in \plays_{\graph_\jf}$ corresponds to a $\jf$-branch $\branch_p$.
These branches are used to define a preference relation for the player $\ptr$ to obtain a game.

\begin{definition}
  The \emph{justification game} $\game_{\js, \interp}$ corresponding to a justification system $\js$ and an interpretation $\interp$ of $\F$ is the antagonistic game $(\graph_\jf, \pref)$ such that for all $p, q \in \plays_\graph$, $p \pref q$ if and only if $\interp(\be(\branch_p)) \leqt \interp(\be(\branch_q))$.
\end{definition}
% \section{Link antagonistic games and justification theory}
We now formally show how a justification system relates to the induced justification game on the semantic level. Intuitively: strategies correspond to justifications and play values dictate the supported value.
We start by the correspondence of strategies for $\ptr$ and justifications.

\paragraph{Strategies for \ptr are justifications}
Intuitively, a strategy for $\ptr$ chooses a rule for every defined fact, which is exactly what a justification does as well. The following propositions make the relation between strategies for \ptr and justifications precise.

\thmwithproof{prop:ptrstrategyisjustification}{Proposition}{proposition}{%
  Let $\sigma$ be a positional (respectively general) strategy for $\ptr$ in the game graph $\graph_\jf$.
  Take $x \in \Fd$.
  The play graph (respectively play tree) of $\sigma$ in $x$ where all the rule symbols are filtered out\footnote{By construction, there are no edges between rule symbols.
  Therefore, filtering, is removing all edges to and from rule symbols and adding edges $(y,z)$ if in the original graph there are edges $(y,r)$ and $(r,z)$ for some rule symbol $r$.}
  is a connected, locally complete graph-like (respectively tree-like) justification with $x$ as root in the justification frame $\jf$.
  We use the symbol $J_\sigma(x)$ for this justification.
}{%
  It is clear that is connected by the definition of play graph/tree.
  Every leaf node is not defined, because plays never stop in defined nodes; hence it is locally complete.
  We take an arbitrary internal node $n$ with label $y$ in $J_\sigma(x)$.
  This means that $y \in \Fd$.
  We prove that the set of labels of the children of $n$ are a case of $y$.
  By definition, $y \in \states_\ptr$, hence there is only one outgoing transition from $n$ in the play graph.
  This transition goes to a node $m$ with a rule symbol $r_{y \gets A} \in \states_\pfa$ as label.
  This means that for all $a \in A$ that there is an edge in the play graph from $m$ to a node $n_a$ labelled $a$.
  This means for all $a \in A$ that $(n,n_a)$ is an edge in $J_\sigma(x)$ and that these are the only edges starting in $n$, which proves that $J_\sigma(x)$ is a justification.
  Since the play graph/tree has a node labelled $x$ as root, so does $J_\sigma(x)$.
}
\begin{example}
  Let $\jf$ be the justification frame from Example \ref{ex:leading} and let $\sigma$ be a positional strategy for $\ptr$ that maps $p$ to $r_{p \gets \tild q}$.
  The justifications $J_\sigma(p)$ and $J_\sigma(\tild p)$ are depicted below from left to right.
  \[
%   \begin{center}
    \begin{tikzpicture}[transform shape]
      \node (p) at (0, 0) {$p$};
      \node (nq) at (1.2, 0) {$\tild q$};
      \tikzstyle{EdgeStyle}=[bend left, style={->}]
      \Edge(p)(nq)
      \Edge(nq)(p)
    \end{tikzpicture}
    \qquad\qquad\qquad
    \begin{tikzpicture}[transform shape]
      \node (np) at (1.2, 0) {$\tild p$};
      \node (q) at (0, 0) {$q$};
      \node (nr) at (2.4, 0) {$\tild r$};
      \tikzstyle{EdgeStyle}=[bend left, style={->}]
      \Edge(np)(q)
      \Edge(q)(np)
      \tikzstyle{EdgeStyle}=[style={->}]
      \Edge(np)(nr)
    \end{tikzpicture}
%   \end{center}
\]
\end{example}

\lemmaInAppendix{
\begin{lemma}\label{lem:ptrstrategyvalue}
  Let $\sigma$ be a positional (respectively general) strategy for $\ptr$.
  For any interpretation $\interp$ and $x \in \Fd$, it holds that $\jval(J_\sigma(x),x,\interp) = \bigwedge_{\tau \in \gstrategies(\pfa)} u(x,\sigma, \tau)$.
\end{lemma}
}
\lemmaInAppendix{
\begin{remark}
  Even for a positional strategy $\sigma$, it is the greatest lower bound over all general strategies $\tau$.
\end{remark}
}

\thmwithproof{prop:ptrstrategyjustificationsareenough}{Proposition}{proposition}{%
  Any connected, locally complete graph-like (respectively tree-like) justification $J$ with root $x$ is equal to some $J_\sigma(x)$ for a positional (respectively general) strategy $\sigma$ for $\ptr$.
}{
  We define a strategy $\sigma$ for $\ptr$.
  Let $p$ be a finite path in $\graph_\jf$ such that $\source(p) = x$ and $\target(p) \in \states_\ptr$.
  Then by filtering out the rule symbols, we get a path $p^*$ in $\jf$.
  It suffices to define $\sigma(p)$ only for paths $p$ such that $p^*$ is a path in $J$.
  Since $J$ is a justification, the set of labels of the children of $\target(p)$ form a case $A$ for the label $y$ of $\target(p)$.
  We define $\sigma(p) = r_{y \gets A}$.
  It is clear that $J_\sigma(x)$ is equal to $J$.
}
The correspondence of justifications and strategies for \ptr raises the question what the role of \pfa is. Intuitively, \ptr  chooses for each node in the justification a set of children corresponding to a rule and \pfa then chooses for each such node a single child. 
This means that \pfa actually chooses a branch in the justification determined by a strategy of \ptr. 
An important remark is that branches in graph-like justifications need not be positional (they can contain both $y \rightarrow u$ and $y \rightarrow v$ for $u \neq v$; in this case, $u$ and $v$ are elements of the same case).
Taking these considerations into account, we can rephrase the supported value of a node in terms of strategies, where \ptr has to play positionally to obtain graph-like justifications but \pfa is always allowed to use general strategies. 
\ignore{

This theorem shows that strategies for $\ptr$ induce justifications.
In order to calculate supported values using these type of justifications,
we need to know how to calculate the value of these justifications, and if the value of these justifications are enough to calculate the supported value.
We start with the former.
Let $\sigma$ be a (positional) strategy for $\ptr$. 
The value of $J_\sigma(x)$ depends on the value of its branches.
A justification can be seen as a set of rules.
A branch in a justification starts in a fact $x$.
Then the next element will be an element of the body of the rule of $x$.
This repeats ad infinitum, or until an open fact is reached.
Choosing an element of the body of a rule is exactly what strategies for $\pfa$ do.
Therefore, every strategy $\tau$ for $\pfa$ corresponds to a $J_\sigma(x)$-branch.
More precisely, the play starting in $x$ consistent with both $\sigma$ and $\tau$ is a sequence of elements in $\F$ and rule symbols.
When removing the rule symbols, a $\jf$-branch is left.
In this case, it is also a $J_\sigma(x)$-branch since the starting play is a path in the play graph (tree) of $\sigma$.
Moreover, every $J_\sigma(x)$-branch can be found this way.
This means that the value of $J_\sigma(x)$ is equal to the greatest lower bound of the value of the plays consistent with $\sigma$ and starting in $x$.
Note that if $\sigma$ is positional, positional strategies for $\pfa$ are not sufficient to cover all branches of $J_\sigma(x)$.
This is because graph-like justifications can contain ``non-positional'' branches,~i.e., branches that contain $y \rightarrow u$ and $y \rightarrow v$ for $u \neq v$.

% By Proposition \ref{prop:supportonlyrootedjustification}, the supported value of $x$ only depends on the value of connected, locally complete justifications with $x$ as root.
% Therefore, if we can prove that every connected, locally complete justification with $x$ is of the form $J_\sigma(x)$ for some strategy $\sigma$ for $\ptr$, then we can calculate the supported value of $x$ only using the values of $x$ in $J_\sigma(x)$ for strategies $\sigma$ for $\ptr$.
% This is exactly what the following propositions states.

As a consequence, the supported value can be computed using play values of strategies.}
To simplify notation, we introduce a utility function defined as follows: $u(x, \sigma, \tau) \coloneqq \interp(\be(\branch_{\play{\graph}(x, \sigma, \tau)}))$.
\thmwithproof{th:positivesupportedvalue}{Theorem}{theorem}{%
  For any $x \in \Fd$ and interpretation $\interp$, the following holds
  \begin{equation*}
    \suppvalue_t(x,\interp) = \bigvee_{\sigma \in \gstrategies(\ptr)} \bigwedge_{\tau \in \gstrategies(\pfa)} u(x,\sigma, \tau) \qquad\text{and}\qquad \suppvalue_g(x,\interp) = \bigvee_{\sigma \in \pstrategies(\ptr)} \bigwedge_{\tau \in \gstrategies(\pfa)} u(x,\sigma, \tau)
  \end{equation*}
}{%
  By Proposition \ref{prop:supportonlyrootedjustification} the supported value of $x$ is equal to the least upper bound of the value of connected, locally complete justifications with root $x$. By Proposition \ref{prop:ptrstrategyjustificationsareenough}, these are exactly the justifications of the form $J_\sigma(x)$ with $\sigma$ a strategy for $\ptr$.
  Using Lemma \ref{lem:ptrstrategyvalue} we then conclude the proof.
}

% The theorem summarises that the supported value is the least upper bound of the values of justifications of the form $J_\sigma(x)$ for a (positionl) strategy $\sigma$ for $\ptr$.
% % The reason that in both formulas $\bigwedge$ quantifies over the general strategies is that the general strategies for $\pfa$ correspond to the branches of $J_\sigma(x)$ irrespective when $\sigma$ is positional or not.
% This is due the fact that graph-like justification can contain ``non-positional'' branches.

\paragraph{Strategies for \pfa are also justifications}
We have seen how strategies for $\ptr$ correspond to justifications and that strategies for $\pfa$ correspond to branches in those justifications.
We used this information to calculate the supported value using play values.
However, in this section we show that also strategies for $\pfa$ correspond to justifications, but for negation of facts.
This is a crucial step towards proving the consistency property, because we can then relate the supported values of $x$ and $\tild x$.
In order to do so, our justification frames needs to be complementary, which, as mentioned above, is a condition satisfied in all practical applications of justification theory.
% Such frames relate the rules for $x$ with the rules for $\tild x$.

Let $\tau$ be a strategy for $\pfa$.
This means that $\tau$ chooses for every rule symbol $y \gets A$ an element from $A$.
This corresponds to a selection function $s$ for $y$.
Using complementarity, this selection function provides a rule for $\tild y$: $\tild y \gets \tild \im(s)$.
Therefore, the strategy $\tau$ indirectly chooses a rule for every $\tild y$, which thus again induces a justification. 
% This is expressed more precisely in the following proposition.

\thmwithproof{prop:pfastrategyisjustification}{Proposition}{proposition}{%
  Let $\jf$ be a complementary justification frame and $\tau$ a positional (respectively general) strategy for $\pfa$ in the game graph $\graph_\jf$ and take $x \in \Fd$.
  The play graph (respectively play tree) of $\tau$ in $x$ where all the rule symbols are filtered out and all node's labels are inverted ($y$ replaced by $\tild y$, remark $\tild(\tild y)=y$), is a connected, locally complete graph-like (respectively tree-like) justification with $\tild x$ as root in the justification frame $\jf$.
  We write $J_\tau(x)$ for this justification.
}{
  By construction $J_\tau(x)$ is connected and locally-complete.
  Take an internal node $n$ of $J_\tau(x)$ with label $y$, hence $y \in \Fd \cap \states_\ptr$.
  This node corresponds to a node $n^*$ with label $\tild y$.
  This means that in the play tree of $\tau$ there is an edge from $n^*$ to a node $m_{\tild y \gets A}^*$ with label $r_{\tild y \gets A}$ for all rules of the form $\tild y \gets A$.
  Since $r_{\tild y \gets A} \in \states_\pfa$, in the play tree of $\tau$ there is exactly one outgoing edge from $m_{\tild y \gets A}^*$ to a node labelled $a \in A$.
  This defines a selection function $s$ for $\tild y$.
  Therefore, in $J_\tau(x)$ the set of labels of the children of $n$ is exactly $\tild\im(s)$.
  Since $\jf$ is a complementary justification frame, it means that $J_\tau(x)$ is a justification.
  That $J_\tau(x)$ has $\tild x$ as root follows directly from the fact that $x$ is a root of the play graph/tree.
}

\begin{example}
  Let $\jf$ be the justification frame from Example \ref{ex:leading}.
  Let $\tau$ be the positional strategy for $\pfa$ that maps $r_{\tild p \gets q,\tild r}$ to $\tild r$.
%  \bart{``any''... er is er maar een? ``the unique''}
%  \simon{Er zijn er twee, $r_{\tild p \gets q,\tild r}$ heeft twee uitgaande bogen}
  The justifications $J_\tau(p)$ and $J_\tau(\tild p)$ are given below from left to right.
%   \begin{center}
\[
    \begin{tikzpicture}[transform shape]
      \node (np) at (1.2, 0) {$\tild p$};
      \node (q) at (0, 0) {$q$};
      \node (nr) at (2.5, 0) {$\tild r$};
      \tikzstyle{EdgeStyle}=[bend left, style={->}]
      \Edge(np)(q)
      \Edge(q)(np)
      \tikzstyle{EdgeStyle}=[style={->}]
      \Edge(np)(nr)
    \end{tikzpicture}
    \qquad\qquad\qquad
    \begin{tikzpicture}[transform shape]
      \node (p) at (0, 0) {$p$};
      \node (r) at (1, 0) {$r$};
      \tikzstyle{EdgeStyle}=[style={->}]
      \Edge(p)(r)
    \end{tikzpicture}
%   \end{center}
\]
%  \bart{Opnieuw:(enkel indien) plaatsgebrek: horizontaal; Of iets ala wrapfigure gebruiken om je tekeningen geen volledige lijnen in te laten nemen.}
\end{example}
\lemmaInAppendix{%
\begin{lemma}
  Let $\jf$ be a complementary justification frame and $\tau$ a positional (respectively general) strategy for $\pfa$, and let $\sigma$ be a general strategy for $\ptr$.
  The branch corresponding to the play $\play{\graph}(x, \sigma, \tau)$ is the negation of a branch in $J_\tau(x)$.
\end{lemma}
}%
\lemmaInAppendix{%
\begin{lemma}
  Let $\jf$ be a complementary justification frame and $\tau$ a positional (respectively general) strategy for $\pfa$.
  Every branch in $J_\tau(x)$ starting with $\tild x$ is the negation of a branch corresponding to the play $\play{\graph}(x, \sigma,\tau)$ for some general strategy $\sigma$ for $\ptr$.
\end{lemma}
}%
\lemmaInAppendix{%
\begin{lemma}\label{lem:pfastrategyvalue}
  Let $\jf$ be a complementary justification frame and $\tau$ a positional (respectively general) strategy for $\pfa$.
  For any $x \in \Fd$ and interpretation $\interp$, it holds that $\jval(J_\tau(x), \tild x, \interp) = \bigwedge_{\sigma \in \gstrategies(\ptr)} \interp(\be(\tild \branch_{\play{\graph}(x, \sigma, \tau)}))$.
\end{lemma}
}%
\lemmaInAppendix{%
\begin{lemma}\label{lem:pfastrategyexactvalue}
  Let $\jf$ be a complementary justification frame and $\be$ a consistent branch evaluation and $\tau$ a positional (respectively general) strategy for $\pfa$.
  For any $x \in \Fd$ and interpretation $\interp$, it holds that $\jval(J_\tau(x), \tild x, \interp) = \tild \bigvee_{\sigma \in \gstrategies(\ptr)} u(x, \sigma, \tau)$.
\end{lemma}
\begin{proof}
  By consistency of $\be$, it follows that $\interp(\be(\tild \branch_{\play{\graph}(x, \sigma, \tau)}))=\tild u(x,\sigma, \tau)$. The result follows by Lemma \ref{lem:pfastrategyvalue} and by noting that $\tild \bigvee_{\sigma \in \gstrategies(\ptr)} u(x, \sigma, \tau) = \bigwedge_{\sigma \in \gstrategies(\ptr)} \tild u(x, \sigma, \tau)$.
\end{proof}
}\thmwithproof{prop:pfastrategyjustificationsareenough}{Proposition}{proposition}{%
  Let $\jf$ be a complementary justification frame.
  Any connected, locally complete graph-like (respectively tree-like) justification with $\tild x$ as root is equal to $J_\tau(x)$ for some positional (respectively general) strategy $\tau$ for $\pfa$.
}{%
  The proof is completely similar to the proof of Proposition \ref{prop:ptrstrategyjustificationsareenough}.
}
\ignore{
To calculate the value of the justification $J_\tau(x)$ in $\tild x$, we need to know the value of its branches.
Let $\sigma$ be a general strategy for $\ptr$.
The play consistent with $\sigma$ and $\tau$ is a path in the play graph (tree) of $\tau$.
This means that filtering out the rule symbols from this play constructs a $\jf$-branch that is the negation of a $J_\tau(x)$-branch.
Moreover, every $J_\tau(x)$-branch is the negation of such a branch constructed from a play consistent with $\tau$ and some general strategy $\sigma$ for $\ptr$.
Therefore, the value of $J_\tau(x)$ is completely determined by the values of plays consistent with $\tau$.
To determine the value of negation of branches, we will need that our branch evaluation $\be$ is consistent.
Recall that $\be$ is consistent if $\be(\tild \branch) = \tild \be(\branch)$ for all branches $\branch$.

We now know how to calculate the value of $\tild x$ in the justification $J_\tau(x)$.
The following proposition states that to determine the supported value of $\tild x$ it suffices to look at the values of justifications of the form $J_\tau(x)$ for a strategy $\tau$ for $\pfa$.

Consequently, the supported value of $\tild x$ can be computed using play values of strategies.
This is formalised in the following theorem.
}
Completely analogously to \cref{th:positivesupportedvalue}, we thus obtain a method to compute the supported value of the negation of a fact in terms of strategies.
%By Proposition \ref{prop:supportonlyrootedjustification}, the supported value of $x$ in $\interp$ only depends on connected, locally complete justifications with $x$ as root.
%Therefore, Proposition \ref{prop:pfastrategyjustificationsareenough} implies the following theorem.

\thmwithproof{th:negativesupportedvalue}{Theorem}{theorem}{%
  Let $\jf$ be a complementary justification frame and $\be$ a consistent branch evaluation.
  For any $x \in \Fd$ and interpretation $\interp$, the following holds
  \begin{equation*}
    \suppvalue_t(\tild x,\interp) = \tild \bigwedge_{\tau \in \gstrategies(\pfa)} \bigvee_{\sigma \in \gstrategies(\ptr)} u(x,\sigma, \tau) \quad \text{and} \quad \suppvalue_g(\tild x,\interp) = \tild \bigwedge_{\tau \in \pstrategies(\pfa)} \bigvee_{\sigma \in \gstrategies(\ptr)} u(x,\sigma, \tau)
  \end{equation*}
}{%
  By Proposition \ref{prop:supportonlyrootedjustification} the supported value of $\tild x$ is equal to the least upper bound of the value of connected, locally complete justifications with $\tild x$ as root.
  By Proposition \ref{prop:pfastrategyjustificationsareenough}, these are exactly the justifications of the form $J_\tau(x)$ with $\tau$ a strategy for $\pfa$.
  Using Lemma \ref{lem:pfastrategyexactvalue} and by noting that $\bigvee_{\tau \in \strategies(\pfa)} \tild \bigvee_{\sigma \in \strategies_g(\ptr)} u(x,\sigma,\tau) = \tild \bigwedge_{\tau \in \strategies(\pfa)} \bigvee_{\sigma \in \strategies_g(\ptr)} u(x,\sigma,\tau)$ we then conclude the proof.
}

% Intuitively, this theorem states that the supported value of $x$ is equal to the least upper bound of the values of $\tild x$ in justifications of the form $J_\tau(x)$ for some strategy $\tau$ for $\pfa$.
% The reason that $\bigvee$ quantifies over general strategies irrespective when you have graph-like or tree-like supported values is that it is possible to have ``non-positional'' branches in the justification $J_\tau(x)$ even if $\tau$ is positional.

\paragraph{Consequences for consistency of justification systems}
\cref{th:positivesupportedvalue} and \cref{th:negativesupportedvalue} relate supported values of $x$ and $\tild x$ to strategies. This provides a partial answer to both research questions as shown in the following two results.

\thmwithproof{prop:consistentbepartialconsistentsupportvalue}{Proposition}{proposition}{%
  Let $\jf$ be a complementary justification frame.
  If $\be$ is consistent, then for any $x \in \Fd$ and interpretation $\interp$ the following holds:
  \begin{equation*}
    \suppvalue_t(x,\interp) \leqt \tild \suppvalue_t(\tild x, \interp)   \qquad\text{and}\qquad   \suppvalue_g(x,\interp) \leqt \tild \suppvalue_g(\tild x, \interp)
  \end{equation*}
}{%
  For the first point, by Theorem \ref{th:positivesupportedvalue} and \ref{th:negativesupportedvalue}, it suffices to prove that $\bigvee_{\sigma \in \gstrategies(\ptr)} \bigwedge_{\tau \in \gstrategies(\pfa)} u(x, \sigma, \tau) \leqt \bigwedge_{\tau \in \gstrategies(\pfa)} \bigvee_{\sigma \in \gstrategies(\ptr)} u(x, \sigma, \tau)$.
  This is a well-known identity.
  
  By Proposition \ref{prop:treebetterthangraph}, we know that $\suppvalue_g(x, \interp) \leqt \suppvalue_t(x, \interp)$.
  Then it also holds that $\suppvalue_g(\tild x, \interp) \leqt \suppvalue_t(\tild x, \interp)$,
  which is the same as $\tild \suppvalue_t(\tild x, \interp) \leqt \tild \suppvalue_g(\tild x, \interp)$.
  By using the first point we get the following chain of inequalities
  \begin{equation*}
  \suppvalue_g(x, \interp) \leqt \suppvalue_t(x, \interp) \leqt \tild \suppvalue_t(\tild x, \interp) \leqt \suppvalue_g(\tild x, \interp),
  \end{equation*}
  which ends the proof.
}
This is one inequality of the consistency.
In the next section, the other direction is researched.
% Furthermore, this partial result can be used to obtain the following theorem which was already announced in the introduction.

%Using Proposition \ref{prop:consistentbepartialconsistentsupportvalue} together with Proposition \ref{prop:treebetterthangraph}, we get the following theorem.

\begin{theorem}\label{th:consistentgraphgivesequaltrees}
  Let $\jf$ be a complementary justification frame.
  If $\be$ is consistent and graph-like justifications are consistent, then tree-like and graph-like justifications are equally powerful.
\end{theorem}
\begin{proof}
  By using Propositions \ref{prop:treebetterthangraph} and \ref{prop:consistentbepartialconsistentsupportvalue} we get the following chain of inequalities:
  \begin{equation*}
  \suppvalue_g(x,\interp) \leqt \suppvalue_t(x,\interp) \leqt \tild \suppvalue_t(\tild x, \interp) \leqt \tild \suppvalue_g(\tild x, \interp).
  \end{equation*}
  By assumption, this chain collapses to only equalities; hence tree-like and graph-like justifications are equally powerful.
\end{proof}
\citet{nmr/MarynissenPBD18} showed that graph-like justifications are consistent for completion, Kripke-Kleene, stable, and well-founded semantics.
Therefore, we automatically have that graph-like and tree-like justification are equally powerful for those branch evaluations.

\section{Consistency of justification systems}

In the previous section we related the supported value under graph-like and tree-like to play values. 
In this section we first show that whenever an optimal pair of strategies exists, the supported value equals the optimal play value. Next we provide sufficient conditions for the existence of optimal pairs of positional strategies.
% as well as showing optimal positional pairs exist in finite justification systems that are monotone and selective.
% 
% The following proposition provides a way to calculate the optimal play value.
%
\lemmaInAppendix{
\begin{proposition}\label{prop:optimalismaximin}
% \thmwithproof{prop:optimalismaximin}{Proposition}{proposition}{%
  If a pair $(\sigma^*, \tau^*)$ of general strategies is optimal,
  then
  \begin{equation*}
  u(x, \sigma^*, \tau^*) = \bigvee_{\sigma \in \gstrategies(\ptr)} \bigwedge_{\tau \in \gstrategies(\pfa)} u(x, \sigma, \tau) = \bigwedge_{\tau \in \gstrategies(\pfa)} \bigvee_{\sigma \in \gstrategies(\ptr)} u(x,\sigma, \tau),
  \end{equation*}
  where $\bigvee$ and $\bigwedge$ are respectively the least upper bound and the greatest lower bound with respect to the order $\leqt$.
  If $(\sigma^*, \tau^*)$ is also positional, then
  \begin{equation*}
  u(x, \sigma^*, \tau^*) = \bigvee_{\sigma \in \pstrategies(\ptr)} \bigwedge_{\tau \in \gstrategies(\pfa)} u(x, \sigma, \tau) = \bigwedge_{\tau \in \pstrategies(\pfa)} \bigvee_{\sigma \in \gstrategies(\ptr)} u(x,\sigma, \tau).
  \end{equation*}
% }{%
\end{proposition}
\begin{proof}
  Assume $(\sigma^*, \tau^*)$ is an optimal pair of strategies.
  By optimality we have that for all general strategies $\sigma$ and $\tau$ that $\play{\graph}(x, \sigma^*, \tau^*) \preftr \play{\graph}(x, \sigma^*, \tau)$.
  This is the same as $u(x, \sigma^*, \tau^*) \leqt u(x, \sigma^*, \tau)$.
  Therefore,
  \begin{equation*}
  u(x,\sigma^*, \tau^*) \leqt \bigwedge_{\tau \in \gstrategies(\pfa)} u(x,\sigma^*, \tau) \leqt \bigvee_{\sigma \in \gstrategies(\ptr)} \bigwedge_{\tau \in \gstrategies(\pfa)} u(x,\sigma, \tau).
  \end{equation*}
  On the other hand, by optimality we have for all general strategies $\sigma$ that $\play{\graph}(x, \sigma, \tau^*) \preftr \play{\graph}(x, \sigma^*, \tau^*)$,
  which is the same as $u(x, \sigma, \tau^*) \leqt u(x, \sigma^*, \tau^*)$.
  Therefore, $\bigvee_{\sigma \in \gstrategies(\ptr)} u(x,\sigma, \tau^*) \leqt u(x,\sigma^*, \tau^*)$.
  We also have that $\bigwedge_{\tau \in \gstrategies(\pfa)} u(x, \sigma, \tau) \leqt u(x, \sigma, \tau^*)$;
  hence that $\bigvee_{\sigma \in \gstrategies(\ptr)} \bigwedge_{\tau \in \gstrategies(\pfa)} u(x, \sigma, \tau) \leqt \bigvee_{\sigma \in \gstrategies(\ptr)} u(x, \sigma, \tau^*) \leqt u(x, \sigma^*,\tau^*)$.
  So we proved that $u(x, \sigma^*, \tau^*) = \bigvee_{\sigma \in \gstrategies(\ptr)} \bigwedge_{\tau \in \gstrategies(\pfa)} u(x, \sigma, \tau)$.
  
  By optimality we have that $u(x, \sigma, \tau^*) \leqt u(x, \sigma^*, \tau^*)$ for all general strategies $\sigma$.
  This means that $\bigvee_{\sigma \in \gstrategies(\ptr)} u(x, \sigma, \tau^*) \leqt u(x, \sigma^*, \tau^*)$.
  We have that
  \begin{equation*}
  \bigwedge_{\tau \in \gstrategies(\pfa)} \bigvee_{\sigma \in \gstrategies(\ptr)} u(x, \sigma, \tau) \leqt \bigvee_{\sigma \in \gstrategies(\ptr)} u(x, \sigma, \tau^*) \leqt u(x, \sigma^*, \tau^*).
  \end{equation*}
  
  Again by optimality we have for all general strategies $\tau$ that $u(x, \sigma^*, \tau^*) \leqt u(x, \sigma^*, \tau)$.
  This means that $u(x,\sigma^*, \tau^*) \leqt \bigwedge_{\tau \in \gstrategies(\pfa)} u(x,\sigma^*, \tau)$.
  We also have that $u(x, \sigma^*, \tau) \leqt \bigvee_{\sigma \in \gstrategies(\ptr)} u(x, \sigma, \tau)$.
  Therefore,
  \begin{equation*}
  u(x, \sigma^*, \tau^*) \leqt \bigwedge_{\tau \in \gstrategies(\pfa)} u(x, \sigma^*, \tau) \leqt \bigwedge_{\tau \in \gstrategies(\pfa)} \bigvee_{\sigma \in \gstrategies(\ptr)} u(x, \sigma, \tau).
  \end{equation*}
  This proves that $u(x, \sigma^*, \tau^*) = \bigwedge_{\tau \in \gstrategies(\pfa)} \bigvee_{\sigma \in \gstrategies(\ptr)} u(x,\sigma, \tau)$.
  
  Now suppose $(\sigma^*, \tau^*)$ is a positional pair of optimal strategies.
  
  By optimality, we have that $\bigvee_{\sigma \in \pstrategies(\ptr)} u(x, \sigma, \tau^*) \leqt u(x,\sigma^*, \tau^*)$ because $\pstrategies(\ptr)$ embeds into $\gstrategies(\ptr)$.
  Therefore,
  \begin{equation*}
  \bigvee_{\sigma \in \pstrategies(\ptr)} \bigwedge_{\tau \in \gstrategies(\pfa)} u(x,\sigma, \tau) \leqt \bigvee_{\sigma \in \pstrategies(\ptr)} u(x,\sigma,\tau^*) \leqt u(x,\sigma^*,\tau^*).
  \end{equation*}
  
  By optimality, $u(x,\sigma^*,\tau^*) \leqt \bigwedge_{\tau \in \gstrategies(\pfa)} u(x,\sigma^*, \tau)$.
  Since $\sigma^*$ is positional, we have that
  \begin{equation*}
  u(x,\sigma^*, \tau^*) \leqt \bigwedge_{\tau \in \gstrategies(\pfa)} u(x,\sigma^*, \tau) \leqt \bigvee_{\sigma \in \pstrategies(\ptr)} \bigwedge_{\tau \in \gstrategies(\pfa)} u(x,\sigma, \tau),
  \end{equation*}
  which proves that $u(x,\sigma^*, \tau^*) = \bigvee_{\sigma \in \pstrategies(\ptr)} \bigwedge_{\tau \in \gstrategies(\pfa)} u(x,\sigma, \tau)$.
  
  By optimality, $\bigvee_{\sigma \in \gstrategies(\ptr)} u(x,\sigma,\tau^*) \leqt u(x,\sigma^*, \tau^*)$.
  Since $\tau^*$ is positional, we have
  \begin{equation*}
  \bigwedge_{\tau \in \pstrategies(\pfa)} \bigvee_{\sigma \in \gstrategies(\ptr)} u(x, \sigma, \tau) \leqt \bigvee_{\sigma \in \gstrategies(\ptr)} u(x, \sigma, \tau^*) \leqt u(x, \sigma^*, \tau^*).
  \end{equation*}
  
  Again by optimality we have that $u(x,\sigma^*, \tau^*) \leqt \bigwedge_{\tau \in \pstrategies(\pfa)} u(x,\sigma^*, \tau)$ since $\pstrategies(\pfa)$ embeds into $\gstrategies(\pfa)$.
  Therefore,
  \begin{equation*}
  u(x, \sigma^*, \tau^*) \leqt \bigwedge_{\tau \in \pstrategies(\pfa)} u(x, \sigma^*, \tau) \leqt \bigwedge_{\tau \in \pstrategies(\pfa)} \bigvee_{\sigma \in \gstrategies(\ptr)} u(x, \sigma, \tau),
  \end{equation*}
  which proves that $u(x, \sigma^*, \tau^*) = \bigwedge_{\tau \in \pstrategies(\pfa)} \bigvee_{\sigma \in \gstrategies(\ptr)} u(x, \sigma, \tau)$
\end{proof}}%

\paragraph{Relating supported values to optimal strategies}
We first focus our attention on the first claim: the relation between supported values of facts (and of their negation) and the value of optimal strategies is formalised in the following two propositions. %The following two propositions achieve the first claim, for facts and their negation respectively. 

% By combining Proposition \ref{prop:optimalismaximin} with Theorem \ref{th:positivesupportedvalue} we get the following result.
\thmwithproof{prop:optimalplayissupportedvalue}{Proposition}{proposition}{
% \begin{proposition}\label{prop:optimalplayissupportedvalue}
  Let $x\in \Fd$ and $\interp$ be an interpretation.
  If $(\sigma^*, \tau^*)$ is optimal, then $\suppvalue_t(x, \interp) = u(x,\sigma^*, \tau^*)$.
  If $(\sigma^*, \tau^*)$ is also positional, then $\suppvalue_g(x,\interp) = u(x,\sigma^*, \tau^*)$.
% \end{proposition}
}{Follows directly by combining Proposition \ref{prop:optimalismaximin} with Theorem \ref{th:positivesupportedvalue}.}%
% Intuitively, this means that the supported value of $x$ in $\interp$ is equal to the value the play dictated by two optimal strategies.
% 
% By combining Proposition \ref{prop:optimalismaximin} with Theorem \ref{th:negativesupportedvalue} we get the following result.
\thmwithproof{prop:negationoptimalplayissupportvalueofnegation}{Proposition}{proposition}{
% 
% \begin{proposition}\label{prop:negationoptimalplayissupportvalueofnegation}
  Let $\jf$ be a complementary justification frame and $\be$ a consistent branch evaluation.
  Let $x \in \Fd$ and $\interp$ an interpretation.
  If $(\sigma^*, \tau^*)$ is optimal, then $\suppvalue_t(\tild x,\interp) = \tild u(x,\sigma^*, \tau^*)$.
  If $(\sigma^*, \tau^*)$ is also positional, then $\suppvalue_g(\tild x, \interp) = \tild u(x, \sigma^*, \tau^*)$.
% \end{proposition}
}{Follows directly by combining Proposition \ref{prop:optimalismaximin} with Theorem \ref{th:negativesupportedvalue}.}
These two propositions relate the play value of optimal pairs of (positional) strategies to the supported value of both $x$ and its negation. As such the existence of optimal strategies is a sufficient condition for the consistency property to hold, which is formalised in the next theorem.

\thmwithproof{th:optimalgivesgoodstuff}{Theorem}{theorem}{%
  Let $\jf$ be a complementary justification frame and $\be$ a consistent branch evaluation.
  If there exists an optimal pair $(\sigma^*, \tau^*)$ of strategies, then for all $x \in \Fd$ and interpretations $\interp$ it holds that $\suppvalue_t(x,\interp) = \tild \suppvalue_t(\tild x, \interp) = u(x,\sigma^*, \tau^*)$.
  If the optimal pair $(\sigma^*,\tau^*)$ is also positional, then for all $x \in \Fd$ and interpretations $\interp$ it holds that $\suppvalue_g(x,\interp) = \tild \suppvalue_g(\tild x, \interp) = u(x,\sigma^*, \tau^*)$.
}{%
  Follows directly from Propositions \ref{prop:optimalplayissupportedvalue} and \ref{prop:negationoptimalplayissupportvalueofnegation}.
}

\paragraph{Existence of optimal pairs of positional strategies}
We now turn our attention to the last piece of the puzzle, namely developing conditions under which optimal pairs of positional strategies exist. The results of the previous section already guarantee that whenever such pairs exist, the consistency property holds, and graph-like and tree-like justifications are equally powerful.
% 
% In the previous subsection, we showed that the existence of optimal strategies implies nice result on the consistency of justification semantics as well on the strength of graph-like vs tree-like justifications.
Therefore, having results for the existence of optimal pairs of positional strategies is paramount.
We want conditions on our branch evaluation that guarantee the existence of positional optimal pairs.
Inspired by \citeANP{hal/GimbertZ04}~\citeyear{hal/GimbertZ04,concur/GimbertZ05}, we have the following two definitions.

\begin{definition}
  $\be$ is monotone if for every two branches $\branch_1$ and $\branch_2$ starting in the same fact $x$ and all finite paths $p$ such that $p \rightarrow x$ is a path the following holds:
  \begin{equation*}
    \interp(\be(\branch_1)) \leqt \interp(\be(\branch_2)) \quad \Rightarrow \quad \interp(\be(p \rightarrow \branch_1)) \leqt \interp(\be(p \rightarrow \branch_2)).
  \end{equation*}
\end{definition}
Monotonicity means that whenever one branch is better than another branch, it does not matter how one arrived there: by adding any single path leading to the start of these two branches in front of the two branches, the inequality is respected.
Intuitively, this is a form of a locality principle: the branch evaluation can decide the value of the two extended branches ($p \rightarrow \branch_1$, $p \rightarrow \branch_2$) based on their joint start ($p$; in which case they have the same value) or based on their tail ($\branch_1$, $\branch_2$; in which case the value of the tail indicates how the value of the extended branches relates). 

In the following definition, we use some more notation and terminology.
A finite loop starting in $x$ is a finite path $p$ starting in $x$ such that $p \rightarrow x$ is a path. If $M$ is a set of loops, $M^*$ denotes the set of (finite) paths obtained by a finite iteration of loops in $M$; $M^\omega$ denotes the set of infinite paths obtained by infinite iterations of loops in $M$.

\begin{definition}
  A branch evaluation $\be$ is selective with respect to $\interp$ if for all finite paths $p$ and facts $x \in \Fd$ such that $p \rightarrow x$ is a path, for all sets $M$, $N$ of finite loops starting in $x$, for all sets $K$ of branches starting in $x$ we have that
  for all branches $\branch$ of the form $p(M \cup N)^*K$ or $p(M \cup N)^\omega$ there exist branches $\branch_*,\branch^* \in pM^\omega \cup pN^\omega \cup pK$ such that
  $\interp(\be(\branch_*)) \leqt \interp(\be(\branch)) \leqt \interp(\be(\branch^*))$.
\end{definition}
Intuitively, selectivity means that choices can be made consistently. This is most easily seen in case $M$, $N$ and $K$ are singletons in which $x$ only occurs as the start. In that case the branch $\branch$ makes (at most) three different choices for $x$: the one leading to the loop in $M$, the one leading to the loop in $N$ and the one leading to $K$. Selectivity then states that the choice for $x$ can be made consistenly, to obtain a branch that is at least as good ($\branch^*$) and one that is at least as bad ($\branch_*$).

% e means that alternating between two kind of loops is not beneficial.
% \simon{More intuition is needed.}
The branch evaluation $\besub{ex}$ from Definition \ref{def:ex} is not selective.
In Example \ref{ex:niko}, if we take $M=\setl{a \rightarrow b}$, $N=\setl{a \rightarrow c}$ and $K = \emptyset$,
then every branch in $N^\omega$ and $M^\omega$ is mapped to $\Fa$, while there are branches in $(M \cup N)^\omega$ that are mapped to $\Un$,~e.g.,~$(a \rightarrow b \rightarrow a \rightarrow c \rightarrow)^\omega$.

\lemmaInAppendix{
The following has been defined in \cite{nmr/MarynissenPBD18}.
\begin{definition}
  A branch evaluation $\be$ is \emph{transitive} if $\be(x_0 \rightarrow x_1 \rightarrow x_2 \rightarrow) = \be(x_1 \rightarrow x_2 \rightarrow \cdots)$ for all branches $x_0 \rightarrow x_1 \rightarrow x_2 \rightarrow \cdots$ with at least three elements.
\end{definition}

\thmwithproof{prop:transitiveismonotone}{Proposition}{proposition}{%
  A transitive branch evaluation is monotone.
}{
  This follows immediately by noting that $\be(p \rightarrow \branch) = \be(\branch)$.
}}
On the other hand, the branch evaluations corresponding to the major logic programming semantics satisfy the two criteria proposed above. 

\thmwithproof{prop:lpismonotoneandselective}{Proposition}{proposition}{%
  The branch evaluations $\besp$, $\bekk$, $\best$, and $\bewf$ are monotone and selective.
}{
  \myparagraph{$\besp$}
  For every finite path $p$ and branch $\branch$, $\interp(\besp(p \rightarrow \branch))$ is either the second element of $p$, the first element of $\branch$, or the second element of $\branch$ depending only on the length of $p$.
  Therefore, $\besp$ is monotone.
  
  For selectivity we only need to look at the case that $p$ is an empty path.
  So take a branch $\branch \in (M \cup N)^*K \cup (M \cup N)^\omega$. Since $\besp(\branch)$ is the second element of $\branch$, the value $\interp(\besp(\branch))$ is already determined when choosing $M$, $N$ or $K$.
  This proves that $\besp$ is selective.
  
  \myparagraph{$\bekk$}
  By Proposition \ref{prop:transitiveismonotone}, it suffices to prove that $\bekk$ is selective.
  Take $\branch \in p(M \cup N)^*K \cup p(M \cup N)^\omega$.
  If $\branch$ is infinite, then $\bekk(\branch) = \Un$.
  If $K$ contains infinite branches, then for any branch $\branch^*$ in $pK$ it holds that $\interp(\bekk(\branch^*)) = \interp(\bekk(\branch))$.
  If $K$ does not contain infinite branches, then either $pM^\omega$ or $pN^\omega$ is not empty.
  Therefore, there is a branch $\branch^* \in pM^\omega \cup pN^\omega$ such that $\interp(\bekk(\branch^*)) = \interp(\bekk(\branch))$.
  
  Assume now that $\branch$ is finite.
  This means that the tail of $\branch$ is contained in $K$.
  Then by transitivity, the result follows.
  
  \myparagraph{$\bewf$}
  By Proposition \ref{prop:transitiveismonotone}, it suffices to prove that $\bewf$ is selective.
  Take $\branch \in p(M \cup N)^*K \cup p(M \cup N)^\omega$.
  If $\branch$ is finite, then a tail $\branch'$ of $\branch$ lies in $K$.
  This means that $\bewf(\branch) = \bewf(\branch') = \bewf(p \rightarrow \branch')$ by transitivity.
  So we can assume that $\branch$ is infinite.
  We can also assume by reasoning like above, that no tail of $\branch$ lies in $K$.
  This means that $\branch \in p(M \cup N)^\omega$.
  If $\branch$ has a tail of equal signs, then there is a loop in either $M$ or $N$ with the same sign.
  Therefore, there is a branch $\branch^*$ in either $pM^\omega$ or $pN^\omega$ with the same value as $\branch$.
  It rests us now the case that $\bewf(\branch) = \Un$.
  We distinguish three cases.
  \begin{enumerate}
    \item[Case 1]
    Either $M$ or $N$ contains a mixed element.
    This means that there is an element in $pM^\omega$ or in $pN^\omega$ that is mapped to $\Un$ under $\bewf$.
    \item[Case 2]
    Every element in $M$ and $N$ is completely positive or completely negative.
    Not all elements in $M$ and $N$ can be completely negative, since $\branch \in p(M \cup N)^\omega$ and $\bewf(\branch) = \Un$.
    Therefore, there is a positive loop $q^+$ and a negative loop $q^-$ in $M \cup N$.
    This means that $\Fa = \interp(\bewf(p(q^+)^\omega)) \leqt \interp(\bewf(\branch)) \leqt \interp(\bewf(p(q^-)^\omega)) = \Tr$.
  \end{enumerate}

  \myparagraph{$\best$}
  Take $\branch_1$ and $\branch_2$ two branches starting in the element $x$ such that $\interp(\best(\branch_1)) \leqt \interp(\best(\branch_2))$.
  Take a finite path $p$ such that $p \rightarrow x$ is a path.
  If $p$ is empty, then the monotonicity condition certainly holds.
  If all elements of $p$ have the same sign, then $\interp(\best(p \rightarrow \branch_1)) = x = \interp(\best(p\rightarrow \branch_2))$ if $x$ has a different sign than $p$ or $\interp(\best(p \rightarrow \branch_1)) = \interp(\best(\branch_1)) \leqt \interp(\best(\branch_2)) \leqt \interp(\best(p \rightarrow \branch_2))$ when $x$ has the same sign as $p$.
  Let $y$ be the first element of $p$ with a different sign than the first element of $p$.
  This means that $\interp(\best(p \rightarrow \branch_1)) = y = \interp(\best(p \rightarrow \branch_2))$.
  This proves that $\best$ is monotone.
  
  Take a branch $\branch \in p(M \cup N)^*K \cup p(M \cup N)^\omega$.
  
  \begin{enumerate}
    \item[Case 1] The branch $\branch$ is infinite and every element of $\branch$ have the same sign.
    Then $\branch$ has a tail in $K$, or there are loops in $M$ or $N$ with the same sign.
    Thus we have a branch $\branch^* \in pM^\omega \cup pN^\omega \cup pK$ such that $\interp(\bewf(\branch^*)) = \interp(\bewf(\branch))$.
    \item[Case 2] The branch $\branch$ is finite and every element except maybe the last has the same sign.
    This means that $\branch$ has a tail $\branch'$ in $K$ such that $\interp(\bewf(\branch)) = \interp(\bewf(p \rightarrow \branch'))$.
    It holds that $p \rightarrow \branch' \in pK$.
    \item[Case 3] The branch $\branch$ has a first sign switch before the last element.
    Assume $y$ is the first sign switch in $\branch$.
    Then either $y$ is in a loop $q_M$ in $M$, a loop $q_N$ in $N$, or in a element $q_K$ of $K$.
    In all cases, we get that $\interp(\bewf(\branch)) = \interp(y) = \interp(\bewf(p \rightarrow (q_M \rightarrow)^\omega)) = \interp(\bewf(p \rightarrow (q_N \rightarrow)^\omega)) = \interp(\bewf(p \rightarrow q_K))$.
  \end{enumerate}
}
\lemmaInAppendix{
\begin{proposition}\label{prop:existsone}
  Let $x \in \states_\ptr$ such that $|x\edges| > 1$.
  Let $x\edges = A_1 \cup A_2$ a partition of non-empty sets.
  Let $\edges_1 = \edges \setminus A_2$
  Let $\edges_2 = \edges \setminus A_1$.
  Let $\game_1$ and $\game_2$ be the subgames of $\game_{\js, \interp}$ with the same states as $\graph_\jf$ and edges $\edges_1$ respectively $\edges_2$.
  If $\be$ is monotone and selective and there are positional optimal pairs $(\sigma_i^*,\tau_i^*)$ in the games $\game_i$ for $i \in \setl{1,2}$,
  then there is an optimal pair $(\sigma^*, \tau^*)$ in $\game_{\js, \interp}$ such that $\sigma^*$ is positional.
\end{proposition}
\begin{proof}
  The proof is heavily based on the proof in \cite{hal/GimbertZ04}.
  
  Without loss of generality, we can assume that $u(x,\sigma_2^*,\tau_2^*) \leqt u(x,\sigma_1^*, \tau_1^*)$.
  We prove that $\sigma^* \coloneqq \sigma_1^*$ is optimal for $\ptr$ in $\graph$, i.e. there is a $\tau^*$ such that $(\sigma_1^*,\tau^*)$ is optimal.
  We now construct a strategy $\tau^*$ for $\pfa$ such that for all strategies $\sigma$ in $\graph$ and all vertices $y$ we have 
  \begin{equation}\label{eq:toprove}
    u(y,\sigma,\tau^*) \leqt u(y,\sigma_1^*, \tau_1^*)
  \end{equation}
  Define a mapping $b:\paths{\graph} \rightarrow \setl{1,2}$.
  Let $p \in \paths{\graph}$ be a finite path in $\graph$.
  Set $b(p) = 1$ if either $p$ does not contain any edge with the source $x$ or the last edge of $p$ with source $x$ belongs to $\graph_1$ ($A_1$).
  Set $b(p) = 2$ if the last edge of $p$ with the source $v$ belongs to $\graph_2$ ($A_2$).
  Define $\tau^*(p) = \tau_i^*(\target(p))$ if $b(p) = i$.
  
  A finite or infinite path $p$ is called homogeneous if $p$ never visits $x$, or if each edge in $p$ with source $x$ belongs to $\edges_1$, or if each edge in $p$ with source $x$ belong to $\edges_2$.
  
  The proof of \ref{eq:toprove} happens in 4 cases.
  
  \begin{enumerate}
    \item[Case 1] $y = x$ and the play $\play{\graph}(y,\sigma, \tau^*)$ is of the form $p_0p_1\ldots p_nq$, where $p_i$ are finite non empty homogeneous paths with source $x$ and $q$ is a homogeneous play with source $x$.
    
    Since $p$ is consistent with $\tau^*$ and $p_i$ are homogeneous, each play $p_i^\omega$ is either consistent with $\tau_1^*$ (if $p_i$ contains only edges of $\graph_1$) or with $\tau_2^*$ (if $p_i$ contains only edges of $\graph_2$).
    Assume without loss of generality that $p_i^\omega$ is consistent with $\tau_2^*$.
    This means there is a $\sigma'$ such that $p_i^\omega = \play{\graph_1}(x,\sigma',\tau_2^*)$; hence $u(p_i^\omega) = u(x,\sigma',\tau_2^*)$.
    By optimality, we get that $u(p_i^\omega) \leqt u(x,\sigma_2^*,\tau_2^*)$.
    By assumption we have that $u(p_i^\omega) \leqt u(x,\sigma_2^*,\tau_2^*) \leqt u(x,\sigma_1^*,\tau_1^*)$.
    
    In the case that $p_i^\omega$ is consistent with $\tau_1^*$, we also get that $u(p_i^\omega) \leqt u(x,\sigma_1^*,\tau_1^*)$.
    Similarly, $u(q) \leqt u(x,\sigma_1^*, \tau_1^*)$.
    By selectivity, we have that $u(x,\sigma, \tau^*)$ is smaller or equal to some $u(p_i^\omega)$ or $u(q)$.
    Then by the results above we have that $u(x,\sigma,\tau^*) \leqt u(x,\sigma_1^*,\tau_1^*)$.
    
    \item[Case 2] $y = x$ and the play $\play{\graph}(y,\sigma, \tau^*)$ is of the form $p_0p_1p_2\ldots$, where $p_i$ is homogeneous non-empty path with source $x$ and $p_ip_{i+1}$ is not homogeneous.
    
    Let $M$ be the set of plays $p_{i_1}p_{i_2}\ldots$ such that $i_1, i_2, \ldots$ are even.
    Let $N$ be the set of plays $p_{i_1}p_{i_2}\ldots$ such that $i_1, i_2, \ldots$ are odd.
    All plays in $M$ are homogeneous and are either all consistent with $\tau_1^*$ or all consistent with $\tau_2^*$.
    If all plays in $M$ are all consistent with $\tau_i^*$, then all plays in $N$ are all consistent with $\tau_{3-i}^*$.
    The plays in $N$ are all homogeneous.
    
    Assume without loss of generality that $M$ is consistent with $\tau_2^*$.
    By optimality of $\tau_2^*$, we have that $u(q) \leqt u(x,\sigma_2^*,\tau_2^*) \leqt u(x,\sigma_1^*,\tau_1^*)$.
    By optimality of $\tau_1^*$, we have that $u(q) \leqt u(x,\sigma_1^*,\tau_1^*)$.
    By selectivity, we know that $u(x,\sigma,\tau^*)$ is smaller than or equal to $u(q)$ for some $q \in M \cup N$.
    Therefore, $u(x,\sigma,\tau^*) \leqt u(x,\sigma_1^*,\tau_1^*)$.
    
    \item[Case 3] $y \neq x$ and the play $p \coloneqq \play{\graph}(y,\sigma,\tau^*)$ never passes through $x$.
    Let $\sigma_1$ be the restriction of $\sigma$ to $\graph_1$.
    Then $p$ is consistent with $\sigma_1$ and $\tau_1^*$ since $b(q) = 1$ for any initial segment of $p$.
    By optimality of $\tau_1^*$ we get $u(p) =u(y,\sigma_1,\tau_1^*) \leqt u(y,\sigma_1^*, \tau_1^*)$.
    
    \item[Case 4] $y \neq x$ and the play $p \coloneqq \play{\graph}(y,\sigma,\tau^*)$ passes at least once the vertex $x$.
    We can factorise $p$ as $rq$ where $r$ is the initial segment of $p$ that does not contain $x$ and the source of $q$ is $x$.
    Let $q^*$ be the play $\play{\graph_1}(x,\sigma_1^*,\tau_1^*)$.
    Of course $u(q^*) = u(x,\sigma_1^*,\tau_1^*)$.
    The play $q$ is consistent with $\tau^*$ with source $x$, and thus by applying case 1 or case 2, we know that $u(q) \leqt u(x,\sigma_1^*,\tau_1^*)$.
    This means that $u(q) \leqt u(q^*)$.
    By monotonicity, we get that $u(p) = u(rq) \leqt u(rq^*)$.
    
    The play $rq^*$ starts in $y$ and is consistent with $\tau_1^*$.
    By optimality of $\tau_1^*$ we get that $u(rq^*) \leqt u(y,\sigma_1^*, \tau_1^*)$.
    This proves that $u(p) \leqt u(y,\sigma_1^*, \tau_1^*)$.
  \end{enumerate}
  
  So we have proven that $u(y,\sigma,\tau^*) \leqt u(y,\sigma_1^*, \tau_1^*)$ for all $\sigma$.
  The play $\play{\graph}(y,\sigma^*,\tau^*)$ does not have a transition in $A_2$ and thus is in $\graph_1$; hence $\play{\graph}(y,\sigma^*,\tau^*) = \play{\graph}(y,\sigma_1^*, \tau_1^*)$.
  So we have proven that
  $u(y,\sigma,\tau^*) \leqt u(y,\sigma^*, \tau^*)$ for all $\sigma$.
  This is the left-hand side of the optimality equation.
  
  Take a general strategy $\tau$.
  Let $\tau_1$ be the restriction of $\tau$ to $\graph_1$.
  Since the play $\play{\graph}(y,\sigma^*, \tau)$ does not have a transition in $A_2$ since $\sigma^*=\sigma_1^*$, we have that this play is in $\graph_1$.
  Therefore, $u(y,\sigma^*, \tau) = u(y,\sigma_1^*, \tau_1)$.
  Now, by optimality of $(\sigma_1^*,\tau_1^*)$ we have that $u(y,\sigma_1^*, \tau_1^*) \leqt u(y,\sigma_1^*, \tau_1)$,
  which shows that
  $u(y,\sigma^*, \tau^*) \leqt u(y,\sigma^*, \tau)$.
  This concludes the proof that $(\sigma^*, \tau^*)$ is an optimal pair of strategies, where $\sigma^*$ is positional. 
\end{proof}
Completely analogous to Proposition \ref{prop:existsone}
\begin{proposition}\label{prop:existstwo}
  Let $x \in \states_\pfa$ such that $|x\edges| > 1$.
  Let $x\edges = A_1 \cup A_2$ a partition of non-empty sets.
  Let $\edges_1 = \edges \setminus A_2$
  Let $\edges_2 = \edges \setminus A_1$.
  Let $\graph_1=(\states,\states_\ptr,\states_\pfa,\edges_1)$ and $\graph_2=(\states,\states_\ptr,\states_\pfa,\edges_2)$ be two subgraphs of $\graph$.
  If there are positional optimal pairs $(\sigma_i^*,\tau_i^*)$ in the games $\graph_i$ for $i \in \setl{1,2}$,
  then there is an optimal pair $(\sigma^*, \tau^*)$ in $\graph$ such that $\tau^*$ is positional.
  
  Let $x \in \states_\pfa$ such that $|x\edges| > 1$.
  Let $x\edges = A_1 \cup A_2$ a partition of non-empty sets.
  Let $\edges_1 = \edges \setminus A_2$
  Let $\edges_2 = \edges \setminus A_1$.
  Let $\game_1$ and $\game_2$ be the subgames of $\game_{\js, \interp}$ with the same states as $\graph_\jf$ and edges $\edges_1$ respectively $\edges_2$.
  If $\be$ is monotone and selective and there are positional optimal pairs $(\sigma_i^*,\tau_i^*)$ in the games $\game_i$ for $i \in \setl{1,2}$,
  then there is an optimal pair $(\sigma^*, \tau^*)$ in $\game_{\js, \interp}$ such that $\tau^*$ is positional.
\end{proposition}
}
If every state owned by $\ptr$ has at most one outgoing transition, then there is only one strategy $\sigma$ for $\ptr$.
This strategy is positional and there is a general strategy $\tau$ for $\pfa$ such that $(\sigma,\tau)$ is optimal.
So we can assume there is a state $x \in \states_\ptr$ with more than one outgoing transition.
If we take a partition $\setl{A_1, A_2}$ of non-empty sets of $x\edges$, then we can form the game graphs $\graph_1$ and $\graph_2$ with the states of the original graph and $\edges \setminus A_2$ respectively $\edges \setminus A_1$ for the edges.
Using the preference relation of the original game, we get two smaller games $\game_1$ and $\game_2$.
Assume by induction that $\game_i$ has an optimal pair $(\sigma_i^*,\tau_i^*)$ of positional strategies for each $i\in\setl{1,2}$.
Without loss of generality we have that $u(x,\sigma_2^*,\tau_2^*) \leqt u(x,\sigma_1^*,\tau_1^*)$.
It turns out that if $\be$ is monotone and selective, then $\sigma_1^*$ is also an optimal strategy for the original game, meaning that there is a general strategy $\tau$ for $\pfa$ such that $(\sigma_1^*,\tau)$ is an optimal pair of strategies.
Therefore, if the game is finite, we can perform finite induction to obtain that there exists a positional strategy $\sigma^*$ for $\ptr$ and a general strategy $\tau$ for $\pfa$ such that $(\sigma^*,\tau)$ is optimal.
By a similar reasoning, there is a positional strategy $\tau^*$ for $\pfa$ and a general strategy $\sigma$ for $\ptr$ such that $(\sigma, \tau^*)$ is optimal.
This means that $(\sigma^*,\tau^*)$ is also an optimal pair, and both strategies are positional.
This paragraph essentially sketched the proof of the following theorem.
 
\thmwithproof{thm:main}{Theorem}{theorem}{%
  If $\game_{\js, \interp}$ is finite and $\be$ is monotone and selective, then there is a positional optimal pair.
}{
  We prove this by induction on the number edges.
  As base case we have games where every state has at most one outgoing edge.
  Because then every pair of strategies for $\ptr$ and $\pfa$ is optimal and positional.
  So assume now by induction that every subgame with fewer edges has an optimal pair of positional strategies.
  If every state owned by $\ptr$ has at most one outgoing edge, then there is only one strategy $\sigma^*$ for $\ptr$.
  This strategy is positional and there exists a general strategy $\tau$ such that $(\sigma^*,\tau)$ is optimal.
  If there is a state $x$ owned by $\ptr$ such that $|x\edges| > 1$, then by using the induction assumption and Proposition \ref{prop:existsone} there exist a positional strategy $\sigma^*$ and a general strategy $\tau$ such that $(\sigma^*,\tau)$ is optimal.
  
  A similar reasoning can be done using Proposition \ref{prop:existstwo}, to get a general strategy $\sigma$ and a positional strategy $\tau^*$ such that $(\sigma, \tau^*)$ is optimal.
  Therefore, also $(\sigma^*,\tau^*)$ is optimal, which concludes the proof.%
}
\begin{corollary}
  If $\js$ is finite and $\be$ is monotone and selective, then $\suppvalue_g(\tild x,\interp) = \tild \suppvalue_g(x,\interp)$ for all $x \in \Fd$ and interpretations $\interp$.
\end{corollary}
For the major logic programming semantics, this means that for finite logic programs, the supported value is consistent and that graph-like and tree-like justifications are equally powerful.
\section{Conclusion}

In this paper, we studied the relation between justification theory and game theory to answer two concrete questions about justification systems: when are graph-like and tree-like justifications equally powerful, and when does the supported value commute with negation? 
All of our work was developed in the general setting except for one (crucial) result: in order to prove the consistency property for graph-like justification systems in case of monotonocity and selectivity,  we assumed, following game theory tradition, finiteness. 
This means for instance that while our results can be applied directly to finite ground logic programs, they are not directly applicable to non-ground programs with an infinite grounding. 

Hence, the most obvious direction for future work is generalizing this theorem to the infinite case, thereby possibly extending the definition of selectivity. 
For inspiration to achieve such result, we think that game theory can again provide valuable directions:
there have been result in a different but similar type of games, the Gale-Stewart games \cite{Soare2016}.
In these games, two player alternates picking an element from a fixed set $A$.
A play is a countable infinite sequence of elements in $A$.
A Gale-Stewart game is won by the first player if the infinite sequence is in some fixed payoff set, otherwise the second player wins.
A famous result by \citeNS{am/Martin75} states that if the pay-off set of a Gale-Stewart game is Borel, then the game is determined, meaning that one of the two players has a winning strategy.
In contrast with our justification games, Gale-Stewart games are only two-valued and every play is infinite.
Three-valued antagonistic games over graphs can be split into two two-valued antagonistic games over graphs.
If these two-valued games can be mapped to Gale-Stewart games such that winning strategies match, Martin's Borel determinacy theorem could be used to prove the existence of optimal pairs of general strategies.

Several earlier papers have already established connections between logic programming and games to study properties of logic programs \cite{jlp/Emden86,lpnmr/Blair95,lpar/LoddoC00,apal/GalanakiRW08}.
By establishing a bridge between game theory and justification theory, we developed a mechanism to transfer game-theoretic results to all application domains of justification theory, e.g., to logic programming, abstract argumentation \mycite{AF}, and to nested fixpoint definitions \cite{lpnmr/DeneckerBS15}. 
Furthermore, by using justification theory, the translation immediately works for all common semantics of logic programs. Indeed, the branch evaluation, which determines the semantics, is only used for determining the resulting preference relation.
% For logic programming a similar connection has already been made by \citeNS{}; by doing it in the general context of justification theory, we pave the way for much broader applicability. 
While we used our results here to get results on consistency and coincidence of graph-like and tree-like justifications, the full impact of this connection still remains to be explored.

\bibliographystyle{acmtrans}
\bibliography{krrlib.bib}

\label{lastpage}

\arxivpaper{%
\appendix
\section{Proofs}
\proofs
}

\end{document}